\documentclass{article}

\usepackage[final]{neurips_2024}

\usepackage[utf8]{inputenc} 
\usepackage[T1]{fontenc}    
\usepackage{hyperref}       
\usepackage{url}            
\usepackage{booktabs}       
\usepackage{amsfonts}       
\usepackage{nicefrac}       
\usepackage{microtype}      
\usepackage{xcolor}         
\usepackage{enumitem}
\usepackage{multirow}
\usepackage{graphicx}
\usepackage{xcolor}   
\usepackage{color, colortbl}
\usepackage{bm}
\usepackage{amsmath} 
\usepackage{mathtools}
\usepackage{enumitem}
\usepackage{amssymb}
\usepackage{wrapfig}
\usepackage{algorithm}
\usepackage{algpseudocode}
\usepackage{amsthm}
\DeclareMathOperator*{\argmin}{argmin}
\DeclareMathOperator*{\argmax}{argmax}

\title{Enhancing Zero-Shot Vision Models by Label-Free Prompt Distribution Learning and Bias Correcting}

\author{Xingyu Zhu\textsuperscript{1{$\star$}}\quad
        Beier Zhu\textsuperscript{2{$\star$}}\quad
        Yi Tan\textsuperscript{1}\quad
        Shuo Wang\textsuperscript{1${\dag}$}\quad
        Yanbin Hao\textsuperscript{1}\quad
        Hanwang Zhang\textsuperscript{2} \\
  $^{1}$University of Science and Technology of China \\
  $^{2}$Nanyang Technological University \\
  \texttt{xingyuzhu@mail.ustc.edu.cn}, \quad \texttt{shuowang.edu@gmail.com}\\
  }

\begin{document}

\def\eg{\emph{e.g.}} 
\def\Eg{\emph{E.g}}
\def\ie{\emph{i.e.}} 
\def\Ie{\emph{I.e}}
\def\cf{\emph{cf} } 
\def\Cf{\emph{Cf}}
\def\etc{\emph{etc}} 
\def\etal{\emph{et al.}} 
\def\vs{\emph{vs}}
\def\wrt{w.r.t. } 
\def\dof{d.o.f}
\def\iid{i.i.d} 
\def\wolog{w.l.o.g}

\definecolor{tabhighlight}{HTML}{e5e5e5}

\newcommand{\Ev}{\Phi_{\mathsf{v}}}
\newcommand{\Et}{\Phi_{\mathsf{t}}}
\newcommand{\softmax}{\mathrm{softmax}}
\newcommand{\x}{\mathbf{x}}
\newcommand{\w}{\mathbf{w}}
\newcommand{\bt}{\mathbf{t}}
\newcommand{\e}{\mathbf{e}}
\newcommand{\z}{\mathbf{z}}
\newcommand{\q}{\mathbf{q}}
\newcommand{\y}{\mathbf{y}}
\newcommand{\s}{\mathbf{s}}

\newcommand{\Ps}{\mathbb{P}_s}
\newcommand{\Pt}{\mathbb{P}_t}
\newcommand{\Pp}{\mathbb{P}_p}
\newcommand{\PP}{\mathbb{P}}
\newcommand{\ens}{\mathsf{ens}}
\newcommand{\ours}{\text{Frolic}}

\newcommand{\piA}{
\hat{\pi}_p^\texttt{m1}}
\newcommand{\piB}{\hat{\pi}_p^\texttt{m2}}

\newcommand{\tableCellHeight}{1}
\newcommand{\tabstyle}[1]{
  \setlength{\tabcolsep}{#1}
  \renewcommand{\arraystretch}{\tableCellHeight}
  \centering
  \small
}

\newenvironment{customitemize}[1]{%
    \begin{list}{\labelitemi}{%
        \setlength{\leftmargin}{#1} 
    }
}{%
    \end{list}
}

\newtheorem{definition}{Definition}
\newtheorem{theorem}{Theorem}
\newtheorem{assumption}{Assumption}
\newtheorem{lemma}{Lemma}
\newtheorem{proposition}{Proposition}
\newtheorem{corollary}{Corollary}

\newtheoremstyle{restatedlemma}
  {\topsep}       
  {\topsep}       
  {\itshape}      
  {}              
  {\bfseries}     
  {.}             
  {.5em}          
  {\thmname{#1} \thmnumber{#2} (\thmnote{#3})} 

\newtheoremstyle{restatedproposition}
  {\topsep}       
  {\topsep}       
  {\itshape}      
  {}              
  {\bfseries}     
  {.}             
  {.5em}          
  {\thmname{#1} \thmnumber{#2} (\thmnote{#3})} 

\theoremstyle{restatedlemma}
\newtheorem*{restatedlemma}{Restated Lemma}

\theoremstyle{restatedproposition}
\newtheorem*{restatedproposition}{Restated Proposition}

\maketitle
\begin{NoHyper}
\def\thefootnote{$\star$}\footnotetext{Equal contributions}
\def\thefootnote{\dag}\footnotetext{Corresponding author}
\end{NoHyper}


\begin{abstract}
 Vision-language models, such as CLIP, have shown impressive generalization capacities when using appropriate text descriptions. While optimizing prompts on downstream labeled data has proven effective in improving performance, these methods entail labor costs for annotations and are limited by their quality. Additionally, since CLIP is pre-trained on highly imbalanced Web-scale data, it suffers from inherent label bias that leads to suboptimal performance. 
 To tackle the above challenges, we propose a label-\textbf{F}ree p\textbf{ro}mpt distribution \textbf{l}earning and b\textbf{i}as \textbf{c}orrection framework, dubbed as \textbf{Frolic}, which boosts zero-shot performance without the need for labeled data. Specifically, our Frolic learns distributions over prompt prototypes to capture diverse visual representations and adaptively fuses these with the original CLIP through confidence matching.
This fused model is further enhanced by correcting label bias via a label-free logit adjustment. Notably, our method is not only training-free but also circumvents the necessity for hyper-parameter tuning. Extensive experimental results across 16 datasets demonstrate the efficacy of our approach, particularly outperforming the state-of-the-art by an average of $2.6\%$ on 10 datasets with CLIP ViT-B/16 and achieving an average margin of $1.5\%$ on ImageNet and its five distribution shifts with CLIP ViT-B/16. Codes are available in \url{https://github.com/zhuhsingyuu/Frolic}.
\end{abstract}
\section{Introduction}
Vision-language models (VLMs), such as CLIP~\cite{clip}, which are pre-trained on large-scale datasets using contrastive loss, effectively align visual and textual representations within a shared feature space. 
This capability enables the zero-shot inference on downstream tasks through prompting and achieves remarkable performance.
For example, using a selection of 80 hand-crafted prompts, a zero-shot CLIP ViT-B/16 achieves an accuracy of $68.7\%$, and with prompts generated by language models~\cite{CuPL}, the accuracy increases to $69.9\%$. 

The success of zero-shot capabilities heavily relies on the appropriate text descriptions of the classes, which has gained research interest in improving prompts. Recent studies propose learning prompts from a small set of labeled images in the downstream data~\cite{coop,cocoop,prograd}. Among these studies, Lu \etal~\cite{LuLZL022} and Wang \etal~\cite{wang2024a} have found that learning the distribution of diverse prompts, which better captures the variance in visual representations, leads to improved performance. 
Although these methods have achieved significant improvements, they still depend on artificial prior knowledge for labeling downstream data and are limited by the quality of manual annotations, which may restrict the scalability of the original model.

Another significant approach to enhancing zero-shot performance involves correcting the label bias inherent in skewed web-scale pre-training data~\cite{simple_zero,parashar2024neglected,ZhuTSZ23}. This bias leads to highly imbalanced predictions and suboptimal performance. As illustrated in Figure~\ref{fig:motivation}(c), the average predicted probability on ImageNet using ViT-B/16 reveals an imbalanced distribution: the highest class probability exceeds $0.002$, whereas the lowest is below $0.0005$. Existing methods correct this bias by allowing access to a portion of the pre-training data~\cite{simple_zero,parashar2024neglected}, or by using labeled downstream data~\cite{ZhuTSZ23}. However, the pre-training data is often inaccessible due to privacy or copyright concerns, and debiasing without labeled data is challenging.

\begin{figure}
    \centering
    \includegraphics[width=0.95\textwidth]{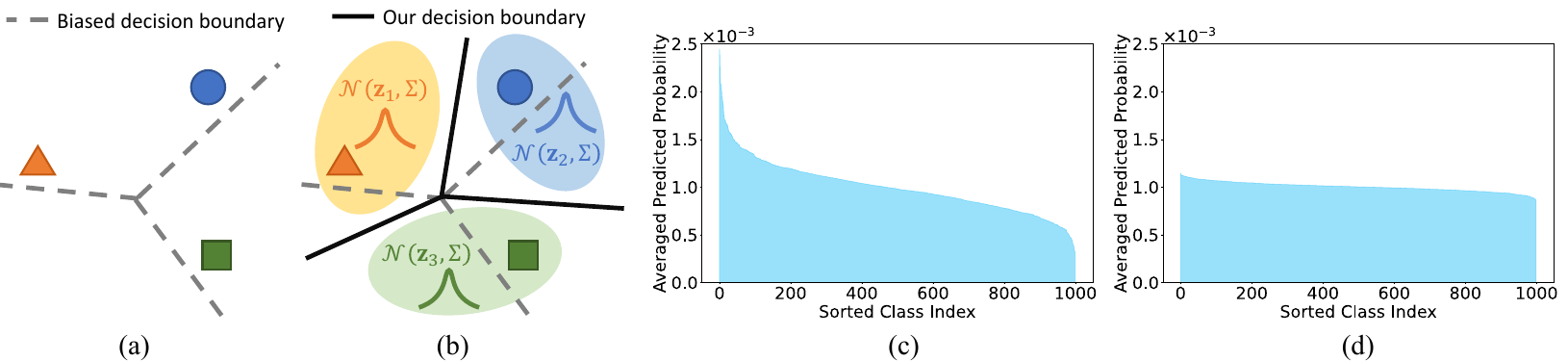}
    \caption{Illustration of prompt distribution learning and label bias correction on ImageNet using CLIP ViT-B/16. (a) Existing zero-shot models~\cite{simple_zero, CuPL}. (b) Our prompt distribution learning (c) Average probability prediction of original CLIP. (d) Average probability prediction of our $\ours$.}
    \label{fig:motivation}
\end{figure}

In this paper, we introduce a label-\textbf{F}ree p\textbf{ro}mpt distribution \textbf{l}earning and b\textbf{i}as \textbf{c}orrection framework, dubbed as \textbf{Frolic}, which eliminates the need for data annotations to enhance zero-shot performance. First, unlike previous methods~\cite{simple_zero,CuPL,coop, wang2023bi,yi2023invariant}, which use a single class prototype for each class to define the decision boundary (as shown in Figure~\ref{fig:motivation}(a)), our approach employs Gaussian distributions to model the varied visual representations of text prototypes, as illustrated in Figure~\ref{fig:motivation}(b). It is worth noting that estimating such a distribution is non-trivial, since classical maximum likelihood estimation requires the annotation of each sample. Fortunately, we demonstrate that it is possible to infer distribution for each class directly from the first and second moments of the marginal distribution of downstream data without label information.
Second, to prevent the use of pre-training data or labeled samples in downstream tasks, we develop a bias estimation mechanism, which transitions the sampling process from the pre-training data distribution to a class-conditional sampling from downstream distribution. By incorporating the estimated label bias into zero-shot models, we can achieve a balanced prediction, as illustrated in Figure~\ref{fig:motivation} (d). 
Furthermore, we explore the possibility of combining the original CLIP predictions with those from the Gaussian-based models to enhance zero-shot performance. To this end, we have developed a confidence-matching technique that dynamically balances the contributions of the two models, eliminating the need for hyperparameter tuning. Notably, our framework is training-free, which enhances both flexibility and ease of implementation.

The main contributions of this work are:
\begin{itemize}[leftmargin=8mm]
\item We enhance zero-shot performance by estimating a distribution over prompt prototypes to capture the variance in visual appearances. We demonstrate that this process can be implemented entirely without labels.

\item We propose a confidence matching technique that fuses the original CLIP model with a Gaussian distribution-based model to further enhance zero-shot performance. This process eliminates the need for hyper-parameter searching, in stark contrast to previous studies.

\item We develop an unsupervised method to correct pre-training label bias. Unlike existing methods that require access to pre-training data, our Proposition~\ref{prop:2} suggests that we can avoid sampling from the pre-training  distribution for estimating and correcting this bias. Instead, our method utilizes only downstream images.

\item We demonstrate the effectiveness of our proposed method $\ours$ by conducting experiments across 16 datasets, which has a consistent and significant improvement over existing baselines. For example, our method surpasses the state-of-the-art zero-shot models by a margin of $2.6\%$ on average with CLIP ViT-B/16. 
\end{itemize}

\section{Related Works}
\noindent\textbf{Zero-shot vision models.} Vision models pre-trained with auxiliary language supervision, such as CLIP~\cite{clip} and OpenCLIP~\cite{ChertiBWWIGSSJ23}, facilitate zero-shot inference through prompting.  
Enhancing zero-shot performance has gained increasing research interest: (1) One approach involves prompt engineering, which includes designing hand-crafted prompts based on human priors~\cite{clip} or automatically generating prompts via language models~\cite{SuS-X}. (2) Another promising approach seeks to improve classifiers, \eg, ZPE~\cite{simple_zero} scores the importance of candidate prompts for prompt ensembling. InMaP~\cite{InMaP} reduces the modality gap between vision and text. Several studies~\cite{TPT, ProAlign} optimize prompt at test time by encouraging consistent predictions across augmented samples.
Our work aims to enhance zero-shot models by learning the prompt distribution and mitigating the pre-training label bias.

\noindent\textbf{Prompt distribution learning.} 
Automatically learning prompts from downstream data has shown potential in improving zero-shot models~\cite{coop,cocoop,prograd}. 
These methods typically optimize prompts via minimizing the classification loss on the target task.
However, as pointed out in Lu~\etal~\cite{LuLZL022}, learning prototype prompts overlook the diversity of visual representations. To this end, they estimate a distribution over the prompts to capture the variance of visual representations.
Recently, Wang~\etal~\cite{wang2024a} propose training-free prompt distribution learning to improve efficiency.
Contrary to existing methods~\cite{LuLZL022} that estimate distributions through supervised approaches, our method circumvents the necessity for labels by inferring the variance of distributions from the statistics of unlabeled data.

\noindent\textbf{Correcting label bias.} Label bias generally occurs in the presence of skewed or imbalanced training data. In response to this challenge, Logit Adjustment (LA)~\cite{TDE,HongHCSKC21,MenonJRJVK21,ZhuTSZ23} has emerged as a prominent technique in long-tailed learning, specifically designed to adjust the decision boundary of classifiers to mitigate label bias. Menon \etal~\cite{MenonJRJVK21} derives the theoretically optimal adjustment for logits. Zhu \etal~\cite{ZhuTSZ23} extents LA to fine-tune zero-shot models by removing the pre-trained label bias. Unlike approaches that rely on the label distribution of the training set~\cite{TDE, HongHCSKC21, MenonJRJVK21,zhu2022cross} or the labels of fine-tuning data~\cite{ZhuTSZ23}, our method adjusts the logits using unlabeled test data.

\section{Methods}
In this section, we present our prompt distribution learning, adaptive fusion, and logit adjustment techniques for adapting zero-shot models.
Without loss of generality, we adopt CLIP~\cite{clip} as our zero-shot model. To begin with, we emphasize three advantages of our framework:

\noindent\textbf{Training-free:} Our $\ours$ is training-free without optimizing the backbone of the zero-shot models, enhancing both flexibility and ease of implementation.

\noindent\textbf{Label-free:} Our method $\ours$ requires no external labeled data, making it suitable for zero-shot scenarios.

\noindent\textbf{No hyper-parameters searching:} Our method $\ours$ eliminates hyper-parameter tuning on validation datasets, in stark contrast to~\cite{wang2024a, tip}

\subsection{Setup}
The zero-shot model consists of a visual encoder $\Ev(\cdot)$ and a text encoder $\Et(\cdot)$. Given a set of unlabeled image data $\{x_i\}_{i=1}^N$ and the unique text set of the class description $\{z_j\}_{j=1}^K$, their visual and text representation can be computed as:
\begin{equation}
    \x_i=\Ev(x_i);\quad \z_j=\Et(z_j),
\end{equation}
where $\x_i$ and $\z_j$ share the same dimension ($\x,\z \in \mathbb{R}^d$). $N$ is the sample size and $K$ is the class size. $\z_j$ can be considered as the prototype for class $j$. With an image $\x$ and all prototypes $\{\z_j\}_{j=1}^K$, zero-shot CLIP predicts the label as:
\begin{equation}\label{eq:fc}
    y = \argmax_j f_{\mathsf{c}}(\x)_j=\argmax_j \z_j^\top\x,
\end{equation}
where $f_{\mathsf{c}}(\x)_j=\z_j^\top\x$ is the score for class $j$. 
\subsection{Label-Free Prompt Distribution Learning}\label{subsec:gda}
In order to express the diverse visual variations, our approach aims to learn the distribution of the class prototypes. Previous studies~\cite{LuLZL022,wang2024a} show that the Gaussian distribution is effective to model the distribution of the CLIP features and achieves impressive improvement. However, these methods require \textit{extra labeled training data}, which is not applicable to our zero-shot setting. 

Specifically, we follow~\cite{wang2024a} to assume $\mathcal{N}(\z_{1:K},\Sigma)$ with identical covariance is the underlying distribution. 
In classical maximum likelihood estimation~\cite{bishop}, the shared covariance $\Sigma$ is computed by averaging the empirical covariances of $K$ classes: 
$\hat{\Sigma}=\frac{1}{K}\sum_{j} \hat{\Sigma}_j$, where $\hat{\Sigma}_j=\frac{1}{|\mathcal{C}_j|-1}\sum_{\x\in \mathcal{C}_j} (\x -\z_j)(\x-\z_j)^\top$. Here, one need the label information of each image to compute $\hat{\Sigma}_j$. Fortunately, to avoid using label information, we can infer $\Sigma$ directly from the expectation and the second order moment of the marginal distribution $\mathbb{P}(\x)$.\footnote{Despite that the modality gap exists between the text and vision space of CLIP models, we can use the unsupervised method from InMaP~\cite{InMaP} to effectively align the two modalities.} Using a Gaussian mixture model with the priors $\{\pi_j\}_{j=1}^K$, $\mathbb{P}(\x)$ is given by:
\begin{equation}
    \mathbb{P}(\x)=\sum_{j=1}^K \pi_j \mathcal{N}(\x;\z_j,\Sigma),\quad \mathcal{N}(\x;\z_j,\Sigma)= \frac{1}{\sqrt{(2\pi)^d|\Sigma|}}\exp\{-\frac{1}{2}(\x - \z_j)^\top\Sigma^{-1}(\x - \z_j)\}
\end{equation}
Denote the second moment of $\x$ as $M$, we have (proof in Section~\ref{sec:proof-LFPDL}):
\begin{equation}\label{eq:S2Sigma}
    M=\Sigma+\sum_j\pi_j\z_j\z_j^\top.
\end{equation}

Denote $\bm{\pi}=[\pi_1,..,\pi_K]^\top$, $Z=[\z_1,..,\z_K]^\top$, and the expectation of $\x$ as $\bm{\mu}$, the prior over the unlabeled data distribution can be estimated by (proof in Section~\ref{sec:proof-pi}):
\begin{equation}\label{eq:est_pi}
    \bm{\pi}=Z^{-1}\bm{\mu}
\end{equation}
We estimate the expectation and the second order moment as 
$\hat{\bm{\mu}}=\frac{1}{N}\sum_{i=1}^N \x_i$ and $\hat{M}=\frac{1}{N}\sum_{i=1}^N \x_i\x_i^\top$, which are unbiased and consistent.
In practice, given that test benchmarks are generally class-balanced, we use a uniform prior over the data distribution, \ie, $\pi_j=\frac{1}{K}$.  Combining with Eq.~\eqref{eq:S2Sigma}, the estimated shared covariance $\hat{\Sigma}$ can be written as: 
 \begin{equation}
      \label{eq:sigma}
      \hat{\Sigma}=\hat{M}-\frac{1}{K}\sum_j\z_j\z_j^\top.
 \end{equation}
Let $\w_j = \hat{\Sigma}^{-1} \z_j$ and $b_j = -\frac{1}{2} \z_j^\top \w_j$, the Gaussian discriminant analysis predicts the label for an image $\x$ as follows (proof in Section~\ref{sec:proof_sg_label}):
\begin{equation}
    \label{eq:sg_label}
    y=\argmax_j f_\mathsf{g}(\x)_j=\argmax_j \w_j^\top\x + b_j
\end{equation}
where $f_{\mathsf{g}}(\x)_j=\w_j^\top\x + b_j$ is the score for class $j$. 

\subsection{Prediction Fusion via Adaptive Calibration.}\label{subsec:cali}
\begin{wrapfigure}{r}{0.4\textwidth}
\vspace{-8mm}
    \centering
\includegraphics[width=0.4\textwidth]{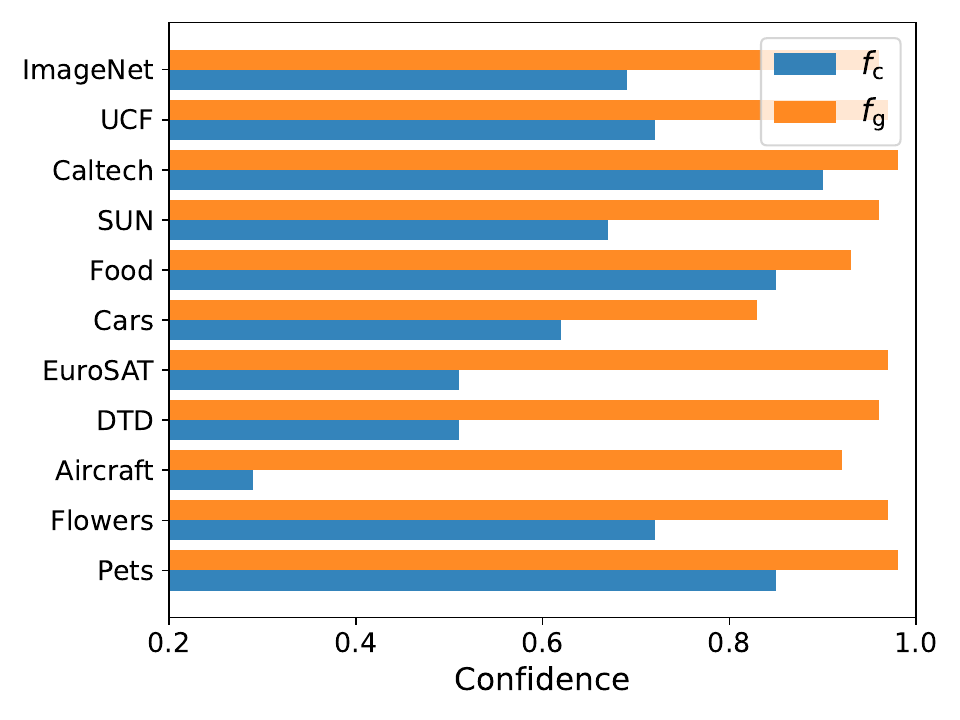}
\vspace{-8mm}
\caption{Comparison of confidence.}
\vspace{-5mm}
\label{fig:orig_confidence}
\end{wrapfigure}
As a rule of thumb, combining the zero-shot predictions with the ones from the learned model can further improve performance for CLIP adaptations~\cite{tip, SuS-X,WortsmanIKLKRLH22,prograd,wang2024a, SSP}.
Previous studies commonly employ a mixing coefficient, $\alpha$, to balance the contributions of two models, \eg, $f(\x) = f_\mathsf{c}(\x) + \alpha f_\mathsf{g}(\x)$. Typically, this hyper-parameter $\alpha$ is optimized on labeled data to maximize accuracy. However, in our context, labels are unavailable, it is not possible to search for the optimal value of $\alpha$. It is imperative to develop a mechanism that balances the prediction fusion without relying on the label.

The key in our prediction fusion lies in aligning the average confidence of the two models. Formally, the average confidence over the dataset $\{\x_i\}_{i=1}^N$ scaled by a temperature $\tau$ is given by the average of the model's probability for its prediction:
\begin{equation}
    \text{conf}(f, \tau)=\frac{1}{N}\sum_{i=1}^N\max_j\softmax(f(\x_i)/\tau)_j.
\end{equation}
Ideally, a model's average confidence should reflect the predicted accuracy, which is called a well-calibrated model. Suppose we have the oracle well-calibrated models, denoted by $f'_\mathsf{c}(\cdot)$ and $f'_\mathsf{g}(\cdot)$, Kumart~\etal~\cite{kumar2022calibrated} prove that the optimal strategy is to fuse the two predictions equally, \ie, $f_\mathsf{f}(\x)=f'_\mathsf{c}(\x)+f'_\mathsf{g}(\x)$. However, as shown in Figure~\ref{fig:orig_confidence},  $f_\mathsf{g}$ is much overconfident than $f_\mathsf{c}$. Let $f_\mathsf{g}(\x)=Cf'_\mathsf{g}(\x)$ for large $C\in \mathbb{R}^+$ (an overconfident model magnifies its logits) and suppose $f_\mathsf{c}(\x)\approx f'_\mathsf{c}(\x)$. 
The fused predictions are given by $f_\mathsf{f}(\x)=Cf'_\mathsf{g}(\x)+f'_\mathsf{c}(\x)$. For very large $C$, $f_\mathsf{f}(\x)$ and $f_\mathsf{g}(\x)$ have the same predictions, \ie, $f_\mathsf{f}(\x)$ is biased towards the $f_\mathsf{g}(\x)$. As we do not have the label to compute accuracy, we cannot apply classical calibration methods~\cite{Recalibration,GuoPSW17} to calibrate $f_\mathsf{g}(\x)$ and $f_\mathsf{c}(\x)$. As our desideratum is to automatically balance the contribution of the two models, we can optimize $\tau_\mathsf{g}$ to make the confidence of $f_\mathsf{g}$ to match up the one of $f_\mathsf{c}$, which circumvent the need of labels:
\begin{equation}\label{eq:searchtau}
    \tau_\mathsf{g}=\argmin_{\tau_\mathsf{g}} \left| \text{conf}(f_\mathsf{g}, \tau_\mathsf{g}) - \text{conf}(f_\mathsf{c}, \tau_\mathsf{c}) \right|
\end{equation}
Specifically, we implement this by binary search, as the confidence monotonically decreases as the temperature increases. $\tau_\mathsf{c}=0.01$ is fixed and learned by CLIP.
The fused logits are given by:
\begin{equation}
    \label{eq:fused_pred}
    f_\mathsf{f}(\x)=f_\mathsf{g}(\x)/\tau_\mathsf{g} + f_\mathsf{c}(\x)/\tau_\mathsf{c}
\end{equation}

\begin{figure}
    \centering
    \begin{minipage}{0.45\linewidth}
        \begin{algorithm}[H]
        \caption{Pipeline of our \ours}  
        \begin{algorithmic}[1]
        \State \textbf{Given}: Unlabeled data $\{\x_i\}_{i=1}^N$, 
        \Statex \hspace{\algorithmicindent} prototypes $\{\z_j\}_{j=1}^K$ and $\tau_\mathsf{c}$
        \State Build $f_\mathsf{c}(\x)_y=\z_y^\top \x$ 
        \State Compute $\hat{\Sigma}=\hat{M}-\frac{1}{K}\sum_j\z_j\z_j^\top$
            \Statex \hspace{\algorithmicindent} where $\hat{M}=\frac{1}{N}\sum_i \x_i\x_i^\top$
        \State Compute $\w_j = \hat{\Sigma}^{-1} \z_j$, $b_j = -\frac{1}{2} \z_j^\top \w_j$
        \State Build $f_{\mathsf{g}}(\x)_y=\w_y^\top\x + b_y$
        \State Search $\tau_\mathsf{g}$ by Eq.~\eqref{eq:searchtau}  
        \State Build $f_\mathsf{f}(\x)=f_\mathsf{g}(\x)/\tau_\mathsf{g} + f_\mathsf{c}(\x)/\tau_\mathsf{c}$
        \State Compute $\hat{\bm{\beta}}$ by Algorithm~\ref{algo:1}
        \State \Return $f_\mathsf{d}(\x) = f_\mathsf{f}(\x) - \ln \hat{\bm{\beta}}$
        \end{algorithmic}\label{algo:all}
        \end{algorithm}
    \end{minipage}
    \hfill
    \begin{minipage}{0.5\linewidth}
        \begin{algorithm}[H]
        \caption{Estimation of $\bm{\beta}$}
        \begin{algorithmic}[1]
        \State \textbf{Given}: Unlabeled data $\{\x_i\}_{i=1}^N$, 
        \Statex \hspace{\algorithmicindent} predictor $f_\mathsf{f}(\cdot)$  and tolerance $\epsilon$.
        \State Initialize $\bm{\beta}^0, f^0_\mathsf{d}$ and $S^0$ by Eq.~\eqref{eq:init}
        \State $t=0$
        \Repeat 
            \State $t = t + 1$
            \State Update $\bm{\beta}^t$ by solving $(S^{t-1}-I)\bm{\beta}^t=\mathbf{0}$
            \State Update $f_\mathsf{d}^t=f_\mathsf{f}-\bm{\beta}^t$
            \State Update $S^t$ from $\s^t_j=\frac{1}{|\mathcal{C}^t_j|}\sum_{\x\in \mathcal{C}^t_j}s(\x),$
            \Statex \hspace{\algorithmicindent} where $\mathcal{C}^t_j$ is assigned by $f_\mathsf{d}^t$
        
        \Until{$\|\bm{\beta}^t-\bm{\beta}^{t-1}\|_1<\epsilon$}
        \State \Return $\hat{\bm{\beta}} = \bm{\beta}^t$
        \end{algorithmic}\label{algo:1}
        \end{algorithm}
    \end{minipage}
\end{figure}

\subsection{Correcting Pre-training Label Bias via Label-Free Logit Adjustment}\label{subsec:debias}
Pre-training datasets typically exhibit a long-tailed concept distribution, leading to biased performance in zero-shot models~\cite{ZhuTSZ23,parashar2024neglected,chen2024catastrophic,simple_zero}. This bias occurs because zero-shot models reflect the posterior probability $\mathbb{P}(y|\x)$ derived from the pre-training distribution. According to Bayes' rule, this posterior probability is influenced by the pre-training label distribution $\mathbb{P}(y)$, as $\mathbb{P}(y|\x) \propto \mathbb{P}(\x|y)\mathbb{P}(y)$. If the prior probability of class $j$ is significantly larger than that of other classes (\eg, $\mathbb{P}(j) \gg \mathbb{P}(i),\ \forall i \in [K], i \neq j$), the predictions will be biased toward class $j$.

Prior research~\cite{MenonJRJVK21,HongHCSKC21} has identified a theoretical optimal solution to address this label bias: let $\beta_y$ denote the prior probability of class $y$, \ie, $\beta_y = \mathbb{P}(y)$. The debiased logit of $f_\mathsf{f}(\x)$ for class $y$ should be (proof in Section~\ref{sec:debias}):
\begin{equation}\label{eq:debias}
f_\mathsf{d}(\x)_y = f_\mathsf{f}(\x)_y - \ln \beta_y.
\end{equation}
Previous methods estimate $\beta$ either by accessing the pre-training data~\cite{parashar2024neglected, simple_zero} or counteract the influence of the prior by optimizing on labeled downstream data~\cite{ZhuTSZ23}. 
However, these approaches are often impractical due to inaccessible pre-training labels due to privacy or copyright concerns or the necessity for labeled downstream data. In this work, we address label bias using only the unlabeled downstream data $\{\x_i\}_{i=1}^N$.

Let $s(\x) = \softmax(f_\mathsf{f}(\x))$ represent the softmax outputs of $f_\mathsf{f}(\x)$, where $s(\x)_y=\hat{\mathbb{P}}(y|\x)$ is the predicted probability for class $y$. Define $\s_j = \mathbb{E}_{\x}[s(\x)|Y=j]$ as the expected posterior probability over the image distribution of class $j$, and let $S = [\s_1,..., \s_K]\in \mathbb{R}^{K\times K}$. We prove that the pre-training label prior $\bm{\beta} = [\beta_1,..., \beta_K]^\top \in \mathbb{R}^K$ must satisfy the following linear equation system:
\begin{equation}\label{eq:qP=q}
    (S-I)\bm{\beta} = \mathbf{0}.
\end{equation} 
\noindent\textbf{Remark.} The key point in Eq.~\eqref{eq:qP=q} is that we avoid sampling from the pre-training data distribution; instead, we sample from $\mathbb{P}(\x|y)$, which is available from the downstream data.
We provide the proof in Section~\ref{sec:proveqP=q} and the numerical power solver for $\bm{\beta}$ in Section~\ref{sec:power}.

We iteratively refine the estimation of $S$ and solve for $\bm{\beta}$ using updated pseudo-labels generated by $f_\mathsf{d}(\x)$. As $f_\mathsf{d}(\x)$ becomes more precise, it yields more accurate pseudo-labels for $\x$, which in turn enhances the accuracy of our estimation of $\bm{\beta}$. Specifically, we initialize
\begin{equation}\label{eq:init}
    \bm{\beta}^0=[1/K,...,1/K]^\top,\ f_\mathsf{d}^{0}=f_\mathsf{f},\  \s_j^0=\frac{1}{|\mathcal{C}^0_j|}\sum_{\x\in \mathcal{C}^0_j}s(\x), \ \text{and}\ S^0=[\s_1^0,...,\s_K^0]
\end{equation}
 where $\x \in \mathcal{C}_j^0$ if $\x$ is classified as $j$ by $f_\mathsf{d}^{0}(\x)$. We proceed by solving for $\bm{\beta}^1$ using Eq.~\eqref{eq:qP=q}, refining $f_\mathsf{d}^1(\x)$ using Eq~\ref{eq:debias} and reassign the pseudo label using $f_\mathsf{d}^1(\x)$ to estimate the updated $\s^1_j$. This process is repeated $t$ times until the relative change in $\bm{\beta}$ satisfies the convergence criterion:
 \begin{equation}\label{eq:iter}
     \frac{\|\bm{\beta}^t-\bm{\beta}^{t-1}\|_1}{\|\bm{\beta}^{t-1}\|_1}=\|\bm{\beta}^t-\bm{\beta}^{t-1}\|_1<\epsilon,\quad \|\bm{\beta}^{t-1}\|=1\ \text{by definition}
 \end{equation}
 where $\epsilon$ is a predefined threshold for relative error tolerance. We summarize the algorithm for solving $\bm{\beta}$ in Algorithm~\ref{algo:1} and provide the overall pipeline in Algorithm~\ref{algo:all}.


\noindent\textbf{Discussion: Comparison with Other Prior Estimation Methods.} 
We compare existing methods for estimating pre-training label priors and demonstrate their in-applicability or flaws in our setting. 

(1) \textit{Explicit method}: the explicit method directly measures the frequency of each class in pre-training data, \eg, $\beta_y=\frac{N_y}{N}$, where $N_y$ is the sample size for class $y$ and $N$ is the total sample size. Most long-tail learning algorithms, \eg, LA and PC~\cite{MenonJRJVK21,HongHCSKC21}, are based on this method because they can access the training data. However, estimating such frequency is complex due to the free-form texts, as opposed to a pre-defined label set. In addition, the pre-training dataset is often inaccessible, making the method inapplicable in our case. 

(2) \textit{Implicit method}: \cite{simple_zero,parashar2024neglected} allow access to a portion of the pre-training data $\mathcal{D}_\mathsf{pt}$ and use the law of total probability to estimate the prior: 
\begin{equation}\label{eq:implicit}
    \beta_y=\mathbb{P}(y)=\int_\x\mathbb{P}_\mathsf{pt}(\x)\mathbb{P}(y|\x)\text{d}\x= \mathbb{E}_{\x\sim \mathbb{P}_\mathsf{pt}(\x)}[\mathbb{P}_\mathsf{pt}(y|\x)] \approx \frac{1}{|\mathcal{D}_\mathsf{pt}|}\sum_{\x\in \mathcal{D}_\mathsf{pt}}\hat{\mathbb{P}}_\mathsf{pt}(y|\x)
\end{equation}
where $\mathbb{\hat{P}}_\mathsf{pt}(y|\x)$ denotes the zero-shot model. However, in our setting, we do have access to the pre-training data or a portion of it. Wang \etal~\cite{wang2022debiased} replace the pre-training data $\mathcal{D}_\mathsf{pt}$ with the downstream data $\mathcal{D}_\mathsf{ds}$ in Eq.~\eqref{eq:implicit} to debias. However, this method neglects the distribution discrepancies between the pre-training and downstream data. In Section~\ref{sec:ablate}, we show that our debiasing significantly outperforms this implicit method. 

(3) \textit{TDE}~\cite{TDE}: Tang \etal~\cite{TDE} debias by removing features along a global direction, retaining only those orthogonal to it. Specifically, the global feature is estimated by $\bar{\x}=\frac{1}{|\mathcal{D}_\mathsf{pt}|}\sum_{\x\in \mathcal{D}_\mathsf{pt}}\x$. Given a test sample $\x$, TDE decomposes it into parallel and orthogonal directions to $\bar{\x}$: $\x=\x_\parallel + \x_\perp$. Then, only the orthogonal component is used for classification: $\hat{\mathbb{P}}_\mathsf{pt}(y|\x_\perp)$. While TDE does not require labels for the samples, we cannot apply it because it requires sampling from the pre-training data. In Section~\ref{sec:ablate}, we replace $\mathcal{D}_\mathsf{pt}$ with downstream data $\mathcal{D}_\mathsf{ds}$ and demonstrate its inferior performance.

(4) \textit{GLA}~\cite{ZhuTSZ23}: Zhu \etal~\cite{ZhuTSZ23} propose to estimate the pre-training prior from the downstream data using the Bayes optimal criterion. The pre-training prior $\bm{\beta}$ is solved by optimizing:
\begin{equation}
    \bm{\beta}=\arg\min_{\bm{\beta}} \mathbb{E}_{(\mathbf{x},y)\sim \mathcal{D}_\mathsf{ds}} [\ell_\mathsf{ce}(f_\mathsf{pt}(\x) - \ln \bm{\beta}, y)],
\end{equation}
where $\ell_\mathsf{ce}$ is the cross-entropy loss and $f_\mathsf{pt}(\x)$ is the logit of the zero-shot model.
While this method circumvents the need for pre-training data access, it is inapplicable because it requires labels for each downstream sample.

\section{Experiments}\label{sec:exp}

\subsection{Setup}\label{sec:setup}
\noindent\textbf{Datasets.} We conduct experiments on 16 image classification benchmarks, covering diverse range categories including generic objects (ImageNet~\cite{Imagenet}, Caltech~\cite{Caltech}), scenes (SUN~\cite{SUN}), textures (DTD~\cite{DTD}), satellite images (EuroSAT~\cite{Eurosat}), actions (UCF~\cite{UCF101}) and fine-grained categories (Pets~\cite{OxfordPet}, Cars~\cite{Cars}, Flowers~\cite{Flower}, Food~\cite{Food-101}, Aircraft~\cite{FGVC}).
Additionally, we evaluate on five ImageNet distribution shifted datasets~\cite{Imagenet}: ImageNetV2 (IN-V2)~\cite{ImageNetV2}, ImageNet-Sketch (IN-Sketch)~\cite{ImageNetSketch}, ImageNet-A (IN-A)~\cite{ImageNetA}, ImageNet-R (IN-R)~\cite{ImageNetR} and ObjectNet~\cite{BarbuMALWGTK19}.

\noindent\textbf{Implementation details.} We adopt CLIP~\cite{clip} ViT-B/16 and ViT-L/14 as our pre-trained models. The default model for ablation studies is CLIP ViT-B/16. We use the same text descriptions as SuS-X~\cite{SuS-X} and CuPL~\cite{CuPL}, and adhere to the InMaP~\cite{InMaP} settings to include all test images. 
$\tau_\mathsf{c}=0.01$ is provided by CLIP. $\epsilon$ in Algorithm~\ref{algo:1} is set to $0.01$. 
All experiments are conducted on a single NVIDIA 3090 GPU if not specified. Note that our algorithm \textit{does not require} any hyper-parameter searching.

\begin{table}[t]
\caption{Comparison of accuracy (\%) on 10 datasets for CLIP ViT-B/16 and ViT-L/14.}
\label{tab:cross_data}
\centering
\tabstyle{5pt}
\begin{tabular}{cl|cccccccccccc}
\toprule
\multicolumn{1}{l}{}                            & Method   & \rotatebox{90}{Pets}     & \rotatebox{90}{Flowers}   & \rotatebox{90}{Aircraft}   & \rotatebox{90}{DTD}      & \rotatebox{90}{EuroSAT}  &\rotatebox{90}{Cars}           & \rotatebox{90}{Food}  & \rotatebox{90}{SUN}       & \rotatebox{90}{Caltech} & \rotatebox{90}{UCF}   & Average  \\ \midrule
\multicolumn{1}{l|}{}                           &CLIP~\cite{clip} & 88.9    & 70.4    & 24.8    & 44.3     & 47.7     & 65.2     & 86.1     & 62.5     & 92.9      & 66.7    & 64.9  \\
\multicolumn{1}{l|}{}                           & TPT \cite{TPT}      & 87.7    & 68.9    & 24.7    & 47.7     & 42.4     & 66.8     & 84.6     & 65.5     & 94.1      & 68.0    & 65.0    \\ 
\multicolumn{1}{l|}{}                           & PromptAlign \cite{ProAlign}      & 90.7    &72.3    &24.8    &47.2     & 47.8     & 68.5     & 86.6     & 67.5     &94.0      & 69.4    & 66.8   \\ 
\multicolumn{1}{l|}{}                           & SuS-X-SD \cite{SuS-X}      & 90.5    &73.8    &28.6    &54.5     &57.4     & 66.1     & 86.0    & 67.7     &93.6      &66.5    &68.4    \\ 
\multicolumn{1}{l|}{}                           &{TDA \cite{TDA}}      & 88.6   & 71.4   & 23.9    & 47.4 & {58.0}  & 67.2     & 86.1     & 67.6     & 94.2      & 70.6   & 67.5 \\ 
\multicolumn{1}{l|}{}                           & {GPT4-Prompt \cite{GPT4V}} & {91.0}     & {74.5}     & {28.0}    & {48.5}      & {48.8}      & {66.8}      & {86.3}      & {65.5}      & {94.6}      & {72.0}     & {67.6} \\ 
\multicolumn{1}{l|}{}                           & CuPL-CLIP~\cite{CuPL} & 92.0    & 73.2    & 27.7    & 54.3     & 52.7    & 66.4     & 86.2    & 68.5     & 94.6      & 70.7   & 68.6  \\ 
\multicolumn{1}{l|}{}                           &  \textbf{Frolic}     & \textbf{92.9}    & \textbf{74.8}     & \textbf{31.5}     & \textbf{56.1}    & \textbf{58.5}  & \textbf{69.1} & \textbf{87.2}  & \textbf{70.8} & \textbf{95.2}      & \textbf{75.2}       & \textbf{71.1}  \\ 
\cmidrule{2-13}
\multicolumn{1}{l|}{}                           & InMaP~\cite{InMaP} & 92.9     & 71.8     &28.4    & 48.0  & 64.1 & 70.6  & 87.7 & 70.5      & 93.1       & 74.0 &70.1  \\ 
\multicolumn{1}{l|}{\multirow{-10}{*}{\rotatebox{90}{\footnotesize{ViT-B/16}}}}                          & \ + \textbf{Frolic}     & \textbf{93.6}    & \textbf{74.3}     & \textbf{31.8}     & \textbf{58.0}    & \textbf{65.3}  & \textbf{71.7} & \textbf{88.2}  & \textbf{72.8} & \textbf{95.4}      & \textbf{75.9}       & \textbf{72.7} \\ 
\midrule
\multicolumn{1}{l|}{}                           &CLIP~\cite{clip} & 93.5    & 79.3   & 32.4    & 53.0     & 58.0     & 76.8     & 91.0     & 67.5     & 94.8      & 74.2    & 72.0   \\
\multicolumn{1}{l|}{}                           & TPT \cite{TPT}         & 93.6    & 76.2    & 31.9    & 55.2     &  51.8    & 77.7     & 88.9     & 70.2  & 95.5     & 74.9   & 71.5 \\ 
\multicolumn{1}{l|}{}                           &{TDA \cite{TDA}}      &93.5    &80.5    &34.7     &56.7  &64.1   &78.3     &90.9     &71.5      &95.9       &76.6    &74.2  \\ 
\multicolumn{1}{l|}{}                           & {GPT4-Prompt \cite{GPT4V}}     & 94.1     & 81.5     & 36.3    & 54.8      & 54.1      & 77.9      & 91.4      & 70.3      & {96.2}      & 80.6    & 73.7  \\ 
\multicolumn{1}{l|}{}                           &CuPL-CLIP~\cite{CuPL}  & 94.3    & 79.8   & 35.5    & 62.7     & 61.2     & 78.0     & 91.3     & 72.4     & 96.7      & 75.9    & 74.7   \\
\multicolumn{1}{l|}{}                           & \textbf{Frolic}    & \textbf{94.9}    & \textbf{82.4}     & \textbf{40.0}     & \textbf{64.1}    & \textbf{66.2}  & \textbf{80.8} & \textbf{91.8}  & \textbf{74.5} & \textbf{97.2}      & \textbf{80.0}       & \textbf{77.1} \\ \cmidrule{2-13}
\multicolumn{1}{l|}{}                           &InMaP~\cite{InMaP}  & 95.2    & 80.7   & 37.6    & 60.2     & 70.6     & 82.5     & 92.2     & 75.0     & 94.9      & 80.4    &  76.9  \\
\multicolumn{1}{l|}{\multirow{-9}{*}{\rotatebox{90}{\footnotesize{ViT-L/14}}}}                          &\ + \textbf{Frolic}    & \textbf{95.4}    & \textbf{81.8}     & \textbf{42.1}     & \textbf{66.9}    & \textbf{71.0}  & \textbf{83.5} & \textbf{92.4}  & \textbf{77.3} & \textbf{97.3}      & \textbf{82.2}       & \textbf{78.9} \\ 
\bottomrule 
\end{tabular}
\end{table}
\subsection{Main Results} 
 We compare our method with several state-of-art methods, including CLIP~\cite{clip}, TPT~\cite{TPT}, PromptAlign~\cite{ProAlign}, SuS-X-DS~\cite{SuS-X}, TDA~\cite{TDA}, GPT4-Prompt~\cite{GPT4V}, CuPL-CLIP~\cite{CuPL}, and InMaP~\cite{InMaP}. Both TPT and TDA utilize a stream of unlabeled test images. For TPT, TDA and InMaP, we produce the results of ViT-L/14 by executing the official released code and maintaining the same hyper-parameters.

\noindent\textbf{Results on 10 datasets.} In Table~\ref{tab:cross_data}, we summarize the accuracy across all datasets, excluding ImageNet and its shifts (denoted as 10-datasets).  Our method consistently shows superior performance across the datasets and backbones, significantly surpassing GPT4-Prompt, which is known for generating high-quality prompts.  
By integrating our method with InMaP, our $\ours$ achieves the highest performance, with an average improvement of $2.6\%$ with ViT-B/16 and $2.0\%$ with ViT-L/14.

\begin{table}[htbp]
\caption{Comparison of accuracy (\%) on ImageNet and its variants for CLIP ViT-B/16 and ViT-L/14.}
\label{tab:imagenet_shift}
\centering
\tabstyle{7pt}
\begin{tabular}{cl|cccccccc}
\toprule
\multicolumn{1}{l}{}                            & Method  & {IN}     & {IN-V2}   & {IN-Sketch}   &{IN-A}      & {IN-R} &{ObjectNet}   & Average \\ \midrule
\multicolumn{1}{l|}{}                           &CLIP~\cite{clip} & 68.7    &62.2     &48.3     &50.6      & 77.7  & 53.5   & 60.1  \\
\multicolumn{1}{l|}{}                           & TPT \cite{TPT}      & 68.9    & 63.4    & 47.9   &54.7      &77.0 & 55.1     & 61.1    \\ 
\multicolumn{1}{l|}{}      & TDA\cite{TDA}      &69.5    & 64.6   & 50.5    & 60.1  & 80.2  &55.1  & 63.3     \\ 
\multicolumn{1}{l|}{}                           & {GPT4-Prompt \cite{GPT4V}} & 68.7     & 62.3     &48.2     &50.6      &77.8   & 53.7   & 60.2  \\  
\multicolumn{1}{l|}{}                           &CuPL-CLIP~\cite{CuPL} & 69.9    &64.4     &49.4     &59.7      &79.5  & 53.7   & 62.7  \\
\multicolumn{1}{l|}{\multirow{-3}{*}{\rotatebox{90}{\footnotesize{ViT-B/16}}}}                         & \textbf{$\ours$}                          & \textbf{70.9}    & \textbf{64.7}     & \textbf{53.3}     & \textbf{60.4}    & \textbf{80.7}  & \textbf{56.6} & \textbf{64.4}  \\  \cmidrule{2-9}
\multicolumn{1}{l|}{}                           & InMaP~\cite{InMaP} & 72.5    & 62.3    & 49.4    & 52.2     & 79.2    & 54.5    & 61.6  \\
\multicolumn{1}{l|}{}                           & \ + \textbf{$\ours$}                          & \textbf{73.3}    & \textbf{63.8}       & \textbf{52.9}    & \textbf{52.8}  &\textbf{79.6}  & \textbf{56.4} &\textbf{63.1} \\ \midrule
\multicolumn{1}{l|}{}                           & CLIP~\cite{clip} & 75.9    &70.2     &59.7    &70.9      & 87.9  &65.5   & 71.6 \\
\multicolumn{1}{l|}{}                           & TPT~\cite{TPT}                & 75.5    & 70.0    & 59.8   & 74.7     & 87.9    & 68.0  & 72.6   \\ 
\multicolumn{1}{l|}{}         & TDA\cite{TDA}      & 76.3    & 71.5    &61.3    & 77.9  & 89.8   & 67.0   & 73.9     \\ 
\multicolumn{1}{l|}{}                           & GPT4-Prompt~\cite{GPT4V} & 75.3     &70.3      &59.9     & 71.2     & 87.8  & 65.7       & 71.7  \\ 
\multicolumn{1}{l|}{}                           &CuPL-CLIP~\cite{CuPL} & 76.2    &71.9     &60.7     &77.9      & 89.6  & 65.7   &73.6  \\
 \multicolumn{1}{l|}{}                           & \textbf{$\ours$}                          & \textbf{77.4}    & \textbf{72.5}     & \textbf{63.1}     & \textbf{78.9}    & \textbf{90.3}  & \textbf{68.7}     & \textbf{75.1}  \\  \cmidrule{2-9}
\multicolumn{1}{l|}{\multirow{-6}{*}{\rotatebox{90}{\footnotesize{ViT-L/14}}}}       & InMaP~\cite{InMaP} & 79.3    & 72.1        &  65.1    &62.5   & 84.8    & 71.0  &72.4 \\
 \multicolumn{1}{l|}{}                           & \ + \textbf{$\ours$}                          & \textbf{79.7}    & \textbf{73.1}     & \textbf{65.7}     & \textbf{64.0}    & \textbf{85.9}        & \textbf{71.7} &\textbf{73.3} \\  \bottomrule
\end{tabular}
\end{table}


\begin{table}[htbp]
\caption{Accuracy (\%) of different models on 10-datasets, ImageNet and its five variant datasets.}
\label{tab:lda_ensemble_debias}
\tabstyle{4pt}
\begin{tabular}{ll|ccc|ccc}
\toprule
 & & \multicolumn{3}{c|}{ViT-B/16}      & \multicolumn{3}{c}{ViT-L/14}      \\  
&\multirow{-2}{*}{Model} & 10-datasets    & ImageNet & \multicolumn{1}{c|}{IN-Variants} & 10-datasets   & ImageNet & IN-Variants \\  \midrule
(1)&$f_\mathsf{c}$                  & 65.1      & 68.7   & 58.5  & 72.0    & 75.9      &72.3 \\ 
(2)&$f_\mathsf{c} - \ln \bm{\beta}$   & 68.4     & 69.7    & 61.2   & 75.1     & 76.2      & 73.4 \\ \midrule
(3)&$f_\mathsf{g}$  & 68.8      &69.8     & 61.3   & 74.7    &  76.0     &73.1 \\ 
(4)&$f_\mathsf{c} + f_\mathsf{g}$       & 66.3   &  68.9   & 59.1      & 72.5     & 76.1    & 72.4        \\ 
(5)&$f_\mathsf{f}=f_\mathsf{c}/\tau_\mathsf{c} + f_\mathsf{g}/\tau_\mathsf{g}$       & 70.4    & 69.8    & 61.9      & 75.5     &  76.9     &73.9         \\ 
\midrule
(6)&$f_\mathsf{d}=f_\mathsf{f}-\ln \bm{\beta}$       & \textbf{71.1}      & \textbf{70.9}      & \textbf{63.1}     & \textbf{77.2}    & \textbf{77.4}     & \textbf{77.4}          \\ 
\bottomrule
\end{tabular}
\end{table}

\noindent\textbf{Results on ImageNet and associated five shifts.}
In Table~\ref{tab:imagenet_shift}, our $\ours$ again surpasses the comparison methods, achieving the average accuracy of $64.4\%$ and $75.1\%$ with ViT-B/16 and ViT-L/14, respectively. Additionally, we observe improvements on the distribution shift datasets: IN-V2, IN-Sketch, IN-A, and IN-R with ViT-B/16, and on IN-A and IN-R with ViT-L/14, when our $\ours$ is combined with InMaP. However, these results still lag behind the original performance of our $\ours$. This discrepancy may stem from the hyper-parameters in InMaP being optimized specifically for ImageNet; applying them unchanged to its shifted datasets could lead to over-fitting. 

\subsection{Ablation Studies and Further Analysis}\label{sec:ablate}
\noindent\textbf{Effectiveness of the prompt distribution learning.} 
In Table~\ref{tab:lda_ensemble_debias} (Row (1) \& (3)), we compare the performance of the original CLIP model $f_\mathsf{c}$ with our prompt distribution learning model $f_\mathsf{g}$. We observe that modeling the underlying distribution of the text prototypes results in notable performance gains. For example, $3.7\%$ accuracy improvement on 10-datasets using ViT-B/16.

\noindent\textbf{Effectiveness of the prediction fusion.} 
As described in Eq.\eqref{eq:fused_pred}, our $\ours$ fuses the original CLIP $f_\mathsf{c}$ and the  
prompt distribution learning model $f_\mathsf{g}$ via confidence matching. 
We compare the simple fusion $f_\mathsf{c}+f_\mathsf{g}$ and our adaptive fusion $f_\mathsf{f}=f_\mathsf{c}/\tau_\mathsf{c} + f_\mathsf{g}/\tau_\mathsf{g}$ in Table~\ref{tab:lda_ensemble_debias} (Row (4) \& (5)). 
We show that our fusion technique outperforms the simple fusion by a large margin. 
Recall that our adaptive fusion method addresses situations where $f_\mathsf{g}$ is more overconfident than $f_\mathsf{c}$.
In Figure~\ref{fig:confidence_acc_curves}, we illustrate the relationship between performance gains over simple fusion—\ie, $\mathsf{Acc}(f_\mathsf{c}/\tau_\mathsf{c} + f_\mathsf{g}/\tau_\mathsf{g}) - \mathsf{Acc}(f_\mathsf{c} + f_\mathsf{g})$—and the confidence difference—\ie, $\left| \text{conf}(f_\mathsf{g}, 1) - \text{conf}(f_\mathsf{c}, \tau_\mathsf{c}) \right|$.
 We present this as a scatter plot where each point represents a dataset, and we have fitted these points with a line. As expected, larger confidence differences correlate with more significant improvements.

 \begin{table}[t]
\caption{Comparison of accuracy ($\%$) between our $\ours$ and other label bias correcting methods for CLIP
ViT-B/16.}
\label{tab:debias_cmp}
    \tabstyle{5.5pt}
    \centering
    \begin{tabular}{lcccccccccccc}
    \toprule
    Model & \rotatebox{90}{Pets}     & \rotatebox{90}{Flowers}   & \rotatebox{90}{Aircraft}   & \rotatebox{90}{DTD}      & \rotatebox{90}{EuroSAT}  &\rotatebox{90}{Cars}           & \rotatebox{90}{Food}  & \rotatebox{90}{SUN}       & \rotatebox{90}{Caltech} & \rotatebox{90}{UCF}  & \rotatebox{90}{ImageNet} & Avg. \\
    \midrule
    CLIP~\cite{clip}   & 89.1  &  71.4  &24.8  &44.3   &47.7 &65.2  &86.1 &62.5 	&92.9 &66.7 &68.7 &65.4 \\
    TDE~\cite{TDE} & 84.1 & 65.8 &27.4  &49.8   &55.3 &60.3  &84.6 &65.5    &91.6 &68.2 &65.9 &65.3 \\
    Implicit        &91.4 &71.4  &30.6  &54.2 &56.8  &66.0 &86.6  &69.5 &93.5    &72.6 &69.8 &69.3 \\
    \textbf{$\ours$}          &\textbf{92.9}  &\textbf{74.8}  &\textbf{31.4}  &\textbf{56.1}  &\textbf{58.5}  &\textbf{69.1} &\textbf{87.1}  &\textbf{70.8}   &\textbf{95.1}    &\textbf{75.2} &\textbf{70.9}   &\textbf{70.9} \\ \midrule
    \textcolor{gray}{Oracle $\ours$}   &\textcolor{gray}{93.1} &\textcolor{gray}{77.5}  &\textcolor{gray}{32.2}  &\textcolor{gray}{57.3} &\textcolor{gray}{59.8}  &\textcolor{gray}{69.8} &\textcolor{gray}{87.4}  &\textcolor{gray}{71.2} &\textcolor{gray}{95.7}   &\textcolor{gray}{76.3} &\textcolor{gray}{71.5} &\textcolor{gray}{71.9} \\
    \bottomrule
    \end{tabular}
\end{table}

\begin{figure*}
\begin{minipage}{0.45\textwidth}
    \centering
    \includegraphics[width=0.8\textwidth]{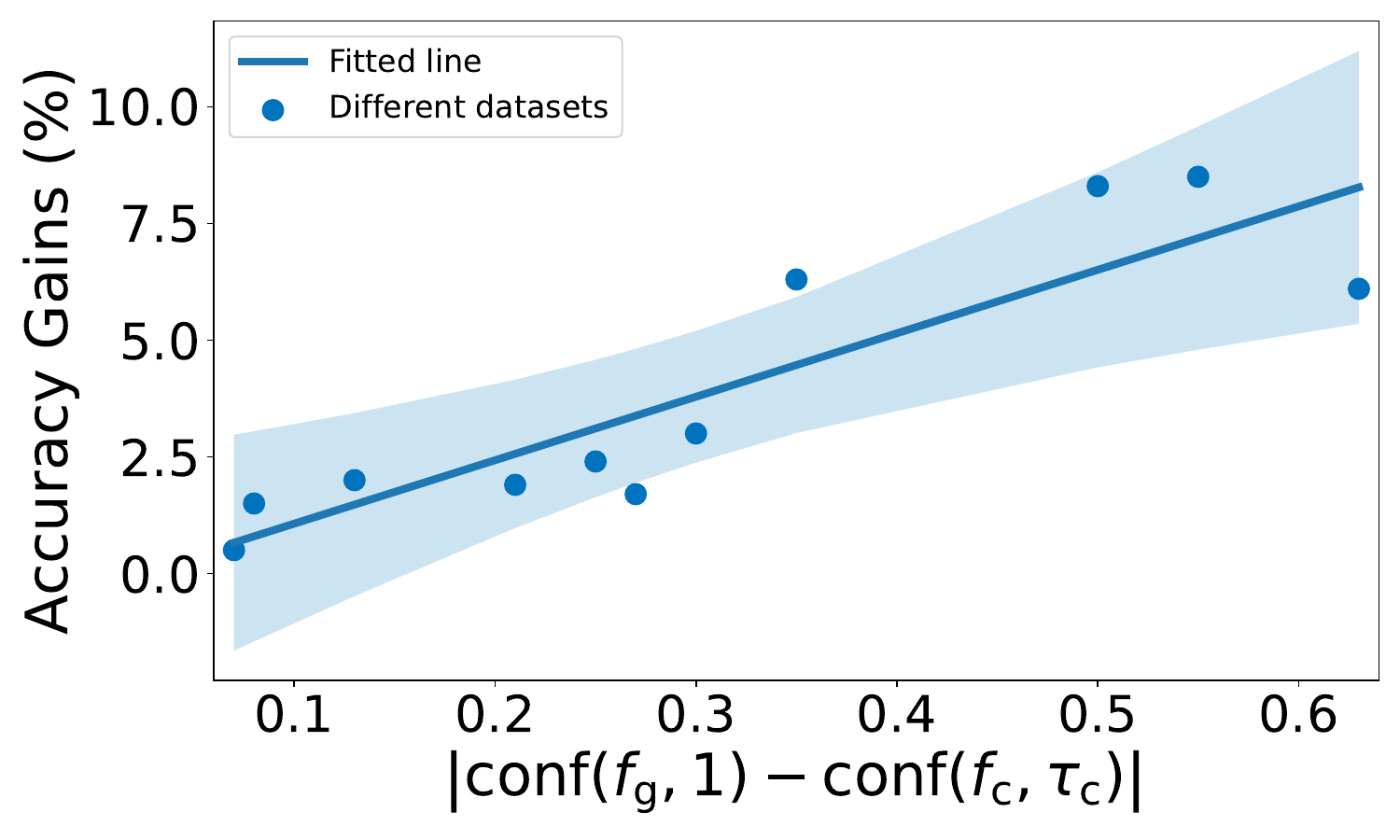}
    \vspace{-0.2cm}
    \caption{Relation between gains and confidence differences.}
    \label{fig:confidence_acc_curves}
\end{minipage}
\hfill
\begin{minipage}{0.45\textwidth}
    \centering
\includegraphics[width=0.8\textwidth]{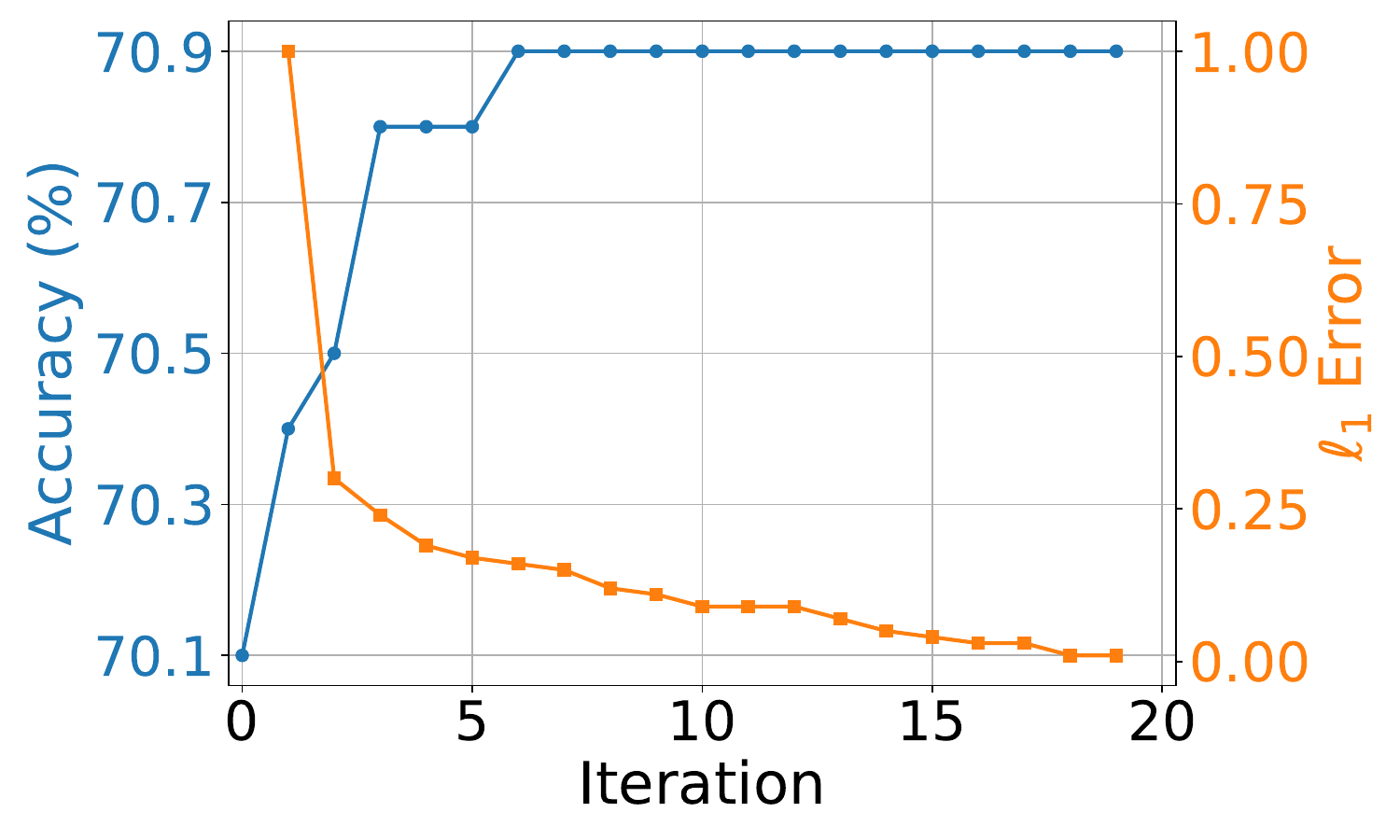}
    \vspace{-0.2cm}
    \caption{Convergence of accuracy and $\ell_1$ error of on ImageNet.}
    \label{fig:debias_acc_curve}
\end{minipage}
\hfill
\vspace{-0.4cm}
\end{figure*}

\noindent\textbf{Effectiveness of the bias correction.}
Row (2) and (6) in Table~\ref{tab:lda_ensemble_debias} demonstrate the effectiveness of our debiasing method, which can further improve the base CLIP model $f_\mathsf{c}$ and the fusion model $f_\mathsf{d}$ across various backbones and datasets. We also compare our debiasing method with other label bias correction methods in Table~\ref{tab:debias_cmp}. The descriptions of TDE~\cite{TDE} and the Implicit method can be found in Section~\ref{subsec:debias}. The results reveal that TDE~\cite{TDE} does not consistently perform well across all datasets. In contrast, while the implicit method using downstream data enhances zero-shot performance, it underperforms compared to our debiasing method, which shows an average gain of $1.6\%$ over the implicit method. To further assess our method's potential, we replaced pseudo-labeling with ground truth labels. The results reveal that the maximum achievable accuracy surpasses our method by $1.0\%$, highlighting the importance of our iterative approach for more accurate pseudo-labeling.

\noindent\textbf{Convergence of Algorithm~\ref{algo:all}.} 
Our method $\ours$, as described in Algorithm~\ref{algo:1}, iteratively solves for the prior $\bm{\beta}$. In Figure~\ref{fig:debias_acc_curve}, we examine the convergence by displaying the errors $\ell_1={\|\bm{\beta}^t-\bm{\beta}^{t-1}\|_1}$ and the accuracy across iterations. We find that the resultant accuracy saturates after only 6 steps, and the relative $\ell_1$ error decreases to less than $\epsilon = 0.01$ after 10 steps.

\noindent\textbf{Comparison with other prompt-based methods.} The popular prompt-based methods, such as CoOp~\cite{coop} and CoCoOp~\cite{cocoop},  require a training procedure with labeled samples while our method does not involve any training.  To ensure a fair comparison, we compare our Frolic with CoOp and CoCoOp on across-dataset results, where the CoOp and CoCoOp are trained only with the labeled samples from the ImageNet dataset and then directly tested on the remaining datasets. The results shown in Table~\ref{tab:pronpt_cmp} demonstrate that our Frolic not only avoids the complexities of training but also exhibits superior generalization performance compared to these methods.
 \begin{table}[hbp]
\caption{Comparison of accuracy ($\%$) between our $\ours$ and prompt-based methods for CLIP ViT-B/16. $*$ denotes our method built upon InMaP~\cite{InMaP}}
\label{tab:pronpt_cmp}
    \tabstyle{5.5pt}
    \centering
    \begin{tabular}{lccccccccccc}
    \toprule
    Model & \rotatebox{90}{ImageNet} & \rotatebox{90}{Pets}     & \rotatebox{90}{Flowers}   & \rotatebox{90}{Aircraft}   & \rotatebox{90}{DTD}      & \rotatebox{90}{EuroSAT}  &\rotatebox{90}{Cars}           & \rotatebox{90}{Food}  & \rotatebox{90}{SUN}       & \rotatebox{90}{Caltech} & \rotatebox{90}{UCF}    \\
    \midrule
    CoOp~\cite{coop}   &71.5  &93.7  &89.1  &64.5   &68.7 &85.3  &18.4 &64.1 	&41.9 &46.3 &66.5  \\
    CoCoOp~\cite{cocoop} & 71.0 &94.4 &90.1  &65.3   &71.8 &86.0  &22.9 &67.3    &45.7 &45.3 &68.2 \\
    \textbf{$\ours$}$^{*}$         &\textbf{73.3}  &\textbf{95.4}  &\textbf{93.6}  &\textbf{71.7}  &\textbf{74.3}  &\textbf{88.2} &\textbf{31.8}  &\textbf{72.8}   &\textbf{58.0}    &\textbf{65.3} &\textbf{75.9}  \\
    \bottomrule
    \end{tabular}
\end{table}

\noindent\textbf{Comparison with adapter-based methods.} The adapter-based methods, \emph{e.g.}, LFA~\cite{LFA} and Tip-Adapter~\cite{tip} boost the CLIP's generalization using labeled training samples. In contrast, our Frolic doesn't require any labeled samples. We evaluate our method with LFA and Tip-Adapert on the ImageNet and its variants dataset, where the LFA and Tip-Adapter only utilize the labeled samples from the ImageNet dataset. The results in Table~\ref{tab:adapter_cmp} show that our method achieves the best performance across all datasets with nearly 3\% improvements in averaged accuracy over LFA.
 \begin{table}[t]
\caption{Comparison of accuracy ($\%$) between our $\ours$ and adapter-based distribution methods for CLIP
ViT-B/16. $*$ denotes our method built upon InMaP~\cite{InMaP}}
\label{tab:adapter_cmp}
    \tabstyle{5.5pt}
    \centering
    \begin{tabular}{lcccccc}
    \toprule
    Model & {IN} & {IN-A}     & {IN-V2}   & {IN-R}   & {IN-Sketch}      & {Average}    \\
    \midrule
    LFA~\cite{LFA}   &72.6  &51.5  &64.7  &76.1   &48.0 & 62.5 \\
    Tip-Adapter~\cite{tip} & 70.5 &49.8  &63.1  &76.9   &48.1 &61.6 \\
    \textbf{$\ours$}$^{*}$         &\textbf{73.3}   &\textbf{52.8}  &\textbf{63.8}  &\textbf{79.6}  &\textbf{52.9} &\textbf{64.4} \\
    \bottomrule
    \end{tabular}
\end{table}

\begin{wraptable}{r}{0.4\textwidth}
\vspace{-0.6cm}
 \caption{Comparison of running time on ImageNet with ViT-B/16.}
 \label{tab:time}
    \tabstyle{3pt}
    \centering
    \begin{tabular}{lrc}
    \toprule
    Model       & Running Time  & Accuracy  \\
    \midrule
    CLIP~\cite{clip} &  6min&  68.7 \\
    TPT \cite{TPT} & 6h & 68.9   \\
    TDA \cite{TDA} & 15min  & 69.5 \\
    \textbf{$\ours$} &  6.5min & 71.1\\
    \bottomrule
    \end{tabular}
\end{wraptable}
\noindent\textbf{Running time.} 
Our method $\ours$ is completely training-free, unlike prompt tuning approaches such as TPT~\cite{TPT} and TDA~\cite{TDA}, which involve back-propagating through an expensive encoder during optimization. We assess the wall-clock time of $\ours$, TPT, and TDA in Table~\ref{tab:time}, using the CLIP ViT-B/16 model on ImageNet. These evaluations are conducted on a single NVIDIA A100 GPU. The results indicate that our method not only requires less time but also delivers superior performance.


\section{Societal Impact, Limitation and Conclusion}\label{sec:conclusion}
\noindent\textbf{Societal impact and limitation.} Models pre-trained on large-scale web-crawled datasets may incorporate knowledge from noisy or malicious samples.

\noindent\textbf{Limitation.} Our approach assumes that the feature representations follow a mixture of Gaussian; however, this assumption may not always hold. On the other hand, the quality and distribution of data used in pre-training can significantly impact the performance of pre-trained models. Our method relies on the capabilities of pre-trained models for downstream tasks, if the pre-trained knowledge differs from the downstream tasks, the efficacy of our method may be limited.
 
\noindent\textbf{Conclusion.} In this work, we propose label-\textbf{F}ree p\textbf{ro}mpt distribution \textbf{l}earning and b\textbf{i}as \textbf{c}orrection, dubbed as \textbf{Frolic}, framework to boost the performance of zero-shot models. Our $\ours$ models each class prototype via a Gaussian distribution and fuses the learned model with the original CLIP~\cite{clip} via confidence matching. The proposed framework further effectively removes the label bias without accessing to the pre-training data. 
Extensive experiments across various datasets demonstrate the effectiveness of our approach.

\begin{ack}
    The work is supported by the National Natural Science Foundation of China (Grants No. 62202439), and the National Research Foundation, Singapore
under its AI Singapore Programme (AISG Award No: AISG2-PhD-2021-01-002). This work is also supported by the advanced computing resources provided by the Supercomputing Center of the USTC.
\end{ack}

\bibliographystyle{plain}
\bibliography{cite}

\begin{thebibliography}{10}

\bibitem{simple_zero}
James~Urquhart Allingham, Jie Ren, Michael~W. Dusenberry, Xiuye Gu, Yin Cui, Dustin Tran, Jeremiah~Zhe Liu, and Balaji Lakshminarayanan.
\newblock A simple zero-shot prompt weighting technique to improve prompt ensembling in text-image models.
\newblock In {\em ICML}, 2023.

\bibitem{BarbuMALWGTK19}
Andrei Barbu, David Mayo, Julian Alverio, William Luo, Christopher Wang, Dan Gutfreund, Josh Tenenbaum, and Boris Katz.
\newblock Objectnet: {A} large-scale bias-controlled dataset for pushing the limits of object recognition models.
\newblock In {\em NeurIPS}, 2019.

\bibitem{bishop}
Christopher~M. Bishop.
\newblock {\em Pattern recognition and machine learning, 5th Edition}.
\newblock Information science and statistics. Springer, 2007.

\bibitem{Food-101}
Lukas Bossard, Matthieu Guillaumin, and Luc~Van Gool.
\newblock Food-101 - mining discriminative components with random forests.
\newblock In {\em ECCV}, 2014.

\bibitem{chen2024catastrophic}
Hao Chen, Bhiksha Raj, Xing Xie, and Jindong Wang.
\newblock On catastrophic inheritance of large foundation models.
\newblock {\em arXiv preprint arXiv:2402.01909}, 2024.

\bibitem{ChertiBWWIGSSJ23}
Mehdi Cherti, Romain Beaumont, Ross Wightman, Mitchell Wortsman, Gabriel Ilharco, Cade Gordon, Christoph Schuhmann, Ludwig Schmidt, and Jenia Jitsev.
\newblock Reproducible scaling laws for contrastive language-image learning.
\newblock In {\em {CVPR}}, 2023.

\bibitem{DTD}
Mircea Cimpoi, Subhransu Maji, Iasonas Kokkinos, Sammy Mohamed, and Andrea Vedaldi.
\newblock Describing textures in the wild.
\newblock In {\em CVPR}, 2014.

\bibitem{Imagenet}
Jia Deng, Wei Dong, Richard Socher, Li{-}Jia Li, Kai Li, and Li~Fei{-}Fei.
\newblock Imagenet: {A} large-scale hierarchical image database.
\newblock In {\em CVPR}, 2009.

\bibitem{Caltech}
Li~Fei{-}Fei, Rob Fergus, and Pietro Perona.
\newblock Learning generative visual models from few training examples: An incremental bayesian approach tested on 101 object categories.
\newblock In {\em {CVPR} Workshops}, 2004.

\bibitem{GuoPSW17}
Chuan Guo, Geoff Pleiss, Yu~Sun, and Kilian~Q. Weinberger.
\newblock On calibration of modern neural networks.
\newblock In {\em ICML}, 2017.

\bibitem{Eurosat}
Patrick Helber, Benjamin Bischke, Andreas Dengel, and Damian Borth.
\newblock Eurosat: {A} novel dataset and deep learning benchmark for land use and land cover classification.
\newblock {\em {IEEE} J. Sel. Top. Appl. Earth Obs. Remote. Sens.}, 12(7):2217--2226, 2019.

\bibitem{ImageNetR}
Dan Hendrycks, Steven Basart, Norman Mu, Saurav Kadavath, Frank Wang, Evan Dorundo, Rahul Desai, Tyler Zhu, Samyak Parajuli, Mike Guo, Dawn Song, Jacob Steinhardt, and Justin Gilmer.
\newblock The many faces of robustness: A critical analysis of out-of-distribution generalization.
\newblock In {\em ICCV}, 2021.

\bibitem{ImageNetA}
Dan Hendrycks, Kevin Zhao, Steven Basart, Jacob Steinhardt, and Dawn Song.
\newblock Natural adversarial examples.
\newblock In {\em CVPR}, 2021.

\bibitem{HongHCSKC21}
Youngkyu Hong, Seungju Han, Kwanghee Choi, Seokjun Seo, Beomsu Kim, and Buru Chang.
\newblock Disentangling label distribution for long-tailed visual recognition.
\newblock In {\em {CVPR}}, 2021.

\bibitem{TDA}
Adilbek Karmanov, Dayan Guan, Shijian Lu, Abdulmotaleb El~Saddik, and Eric Xing.
\newblock Efficient test-time adaptation of vision-language models.
\newblock In {\em CVPR}, 2024.

\bibitem{Cars}
Jonathan Krause, Michael Stark, Jia Deng, and Li~Fei{-}Fei.
\newblock 3d object representations for fine-grained categorization.
\newblock In {\em {ICCV} Workshops}, 2013.

\bibitem{kumar2022calibrated}
Ananya Kumar, Tengyu Ma, Percy Liang, and Aditi Raghunathan.
\newblock Calibrated ensembles can mitigate accuracy tradeoffs under distribution shift.
\newblock In {\em UAI}, 2022.

\bibitem{LuLZL022}
Yuning Lu, Jianzhuang Liu, Yonggang Zhang, Yajing Liu, and Xinmei Tian.
\newblock Prompt distribution learning.
\newblock In {\em CVPR}, 2022.

\bibitem{Recalibration}
Rachel Luo, Shengjia Zhao, Jiaming Song, Jonathan Kuck, Stefano Ermon, and Silvio Savarese.
\newblock Privacy preserving recalibration under domain shift.
\newblock {\em CoRR}, abs/2008.09643, 2020.

\bibitem{FGVC}
Subhransu Maji, Esa Rahtu, Juho Kannala, Matthew~B. Blaschko, and Andrea Vedaldi.
\newblock Fine-grained visual classification of aircraft.
\newblock {\em CoRR}, abs/1306.5151, 2013.

\bibitem{MenonJRJVK21}
Aditya~Krishna Menon, Sadeep Jayasumana, Ankit~Singh Rawat, Himanshu Jain, Andreas Veit, and Sanjiv Kumar.
\newblock Long-tail learning via logit adjustment.
\newblock In {\em ICLR}, 2021.

\bibitem{mises1929praktische}
RV~Mises and Hilda Pollaczek-Geiringer.
\newblock Praktische verfahren der gleichungsaufl{\"o}sung.
\newblock {\em ZAMM-Journal of Applied Mathematics and Mechanics/Zeitschrift f{\"u}r Angewandte Mathematik und Mechanik}, 9(1):58--77, 1929.

\bibitem{Flower}
Maria{-}Elena Nilsback and Andrew Zisserman.
\newblock Automated flower classification over a large number of classes.
\newblock In {\em ICVGIP}, 2008.

\bibitem{LFA}
Yassine Ouali, Adrian Bulat, Brais Mart{\'{\i}}nez, and Georgios Tzimiropoulos.
\newblock Black box few-shot adaptation for vision-language models.
\newblock {\em CoRR}, abs/2304.01752, 2023.

\bibitem{parashar2024neglected}
Shubham Parashar, Zhiqiu Lin, Tian Liu, Xiangjue Dong, Yanan Li, Deva Ramanan, James Caverlee, and Shu Kong.
\newblock The neglected tails of vision-language models.
\newblock {\em arXiv preprint arXiv:2401.12425}, 2024.

\bibitem{OxfordPet}
Omkar~M. Parkhi, Andrea Vedaldi, Andrew Zisserman, and C.~V. Jawahar.
\newblock Cats and dogs.
\newblock In {\em CVPR}, 2012.

\bibitem{CuPL}
Sarah~M. Pratt, Ian Covert, Rosanne Liu, and Ali Farhadi.
\newblock What does a platypus look like? generating customized prompts for zero-shot image classification.
\newblock In {\em ICCV}, 2023.

\bibitem{InMaP}
Qi~Qian, Yuanhong Xu, and Juhua Hu.
\newblock Intra-modal proxy learning for zero-shot visual categorization with {CLIP}.
\newblock In {\em NeurIPS}, 2023.

\bibitem{clip}
Alec Radford, Jong~Wook Kim, Chris Hallacy, Aditya Ramesh, Gabriel Goh, Sandhini Agarwal, Girish Sastry, Amanda Askell, Pamela Mishkin, Jack Clark, Gretchen Krueger, and Ilya Sutskever.
\newblock Learning transferable visual models from natural language supervision.
\newblock In {\em ICML}, 2021.

\bibitem{ImageNetV2}
Benjamin Recht, Rebecca Roelofs, Ludwig Schmidt, and Vaishaal Shankar.
\newblock Do imagenet classifiers generalize to imagenet?
\newblock In {\em {ICML}}, 2019.

\bibitem{ProAlign}
Jameel~Abdul Samadh, Hanan Gani, Noor Hussein, Muhammad~Uzair Khattak, Muzammal Naseer, Fahad~Shahbaz Khan, and Salman~H. Khan.
\newblock Align your prompts: Test-time prompting with distribution alignment for zero-shot generalization.
\newblock In {\em NeurIPS}, 2023.

\bibitem{TPT}
Manli Shu, Weili Nie, De{-}An Huang, Zhiding Yu, Tom Goldstein, Anima Anandkumar, and Chaowei Xiao.
\newblock Test-time prompt tuning for zero-shot generalization in vision-language models.
\newblock In {\em NeurIPS}, 2022.

\bibitem{UCF101}
Khurram Soomro, Amir~Roshan Zamir, and Mubarak Shah.
\newblock {UCF101:} {A} dataset of 101 human actions classes from videos in the wild.
\newblock {\em CoRR}, abs/1212.0402, 2012.

\bibitem{TDE}
Kaihua Tang, Jianqiang Huang, and Hanwang Zhang.
\newblock Long-tailed classification by keeping the good and removing the bad momentum causal effect.
\newblock In {\em NeurIPS}, 2020.

\bibitem{SuS-X}
Vishaal Udandarao, Ankush Gupta, and Samuel Albanie.
\newblock Sus-x: Training-free name-only transfer of vision-language models.
\newblock {\em CoRR}, abs/2211.16198, 2022.

\bibitem{ImageNetSketch}
Haohan Wang, Songwei Ge, Zachary Lipton, and Eric~P Xing.
\newblock Learning robust global representations by penalizing local predictive power.
\newblock In {\em Advances in Neural Information Processing Systems}, pages 10506--10518, 2019.

\bibitem{wang2022debiased}
Xudong Wang, Zhirong Wu, Long Lian, and Stella~X Yu.
\newblock Debiased learning from naturally imbalanced pseudo-labels.
\newblock In {\em CVPR}, 2022.

\bibitem{wang2024a}
Zhengbo Wang, Jian Liang, Lijun Sheng, Ran He, Zilei Wang, and Tieniu Tan.
\newblock A hard-to-beat baseline for training-free {CLIP}-based adaptation.
\newblock In {\em ICLR}, 2024.

\bibitem{wang2023bi}
Zhicai Wang, Yanbin Hao, Tingting Mu, Ouxiang Li, Shuo Wang, and Xiangnan He.
\newblock Bi-directional distribution alignment for transductive zero-shot learning.
\newblock In {\em CVPR}, 2023.

\bibitem{WortsmanIKLKRLH22}
Mitchell Wortsman, Gabriel Ilharco, Jong~Wook Kim, Mike Li, Simon Kornblith, Rebecca Roelofs, Raphael~Gontijo Lopes, Hannaneh Hajishirzi, Ali Farhadi, Hongseok Namkoong, and Ludwig Schmidt.
\newblock Robust fine-tuning of zero-shot models.
\newblock In {\em CVPR}, 2022.

\bibitem{GPT4V}
Wenhao Wu, Huanjin Yao, Mengxi Zhang, Yuxin Song, Wanli Ouyang, and Jingdong Wang.
\newblock Gpt4vis: What can {GPT-4} do for zero-shot visual recognition?
\newblock {\em CoRR}, abs/2311.15732, 2023.

\bibitem{SUN}
Jianxiong Xiao, James Hays, Krista~A. Ehinger, Aude Oliva, and Antonio Torralba.
\newblock {SUN} database: Large-scale scene recognition from abbey to zoo.
\newblock In {\em CVPR}, 2010.

\bibitem{yi2023invariant}
Xuanyu Yi, Jiajun Deng, Qianru Sun, Xian-Sheng Hua, Joo-Hwee Lim, and Hanwang Zhang.
\newblock Invariant training 2d-3d joint hard samples for few-shot point cloud recognition.
\newblock In {\em ICCV}, 2023.

\bibitem{tip}
Renrui Zhang, Wei Zhang, Rongyao Fang, Peng Gao, Kunchang Li, Jifeng Dai, Yu~Qiao, and Hongsheng Li.
\newblock Tip-adapter: Training-free adaption of {CLIP} for few-shot classification.
\newblock In {\em ECCV}, 2022.

\bibitem{cocoop}
Kaiyang Zhou, Jingkang Yang, Chen~Change Loy, and Ziwei Liu.
\newblock Conditional prompt learning for vision-language models.
\newblock In {\em CVPR}, 2022.

\bibitem{coop}
Kaiyang Zhou, Jingkang Yang, Chen~Change Loy, and Ziwei Liu.
\newblock Learning to prompt for vision-language models.
\newblock {\em IJCV}, 130(9):2337--2348, 2022.

\bibitem{prograd}
Beier Zhu, Yulei Niu, Yucheng Han, Yue Wu, and Hanwang Zhang.
\newblock Prompt-aligned gradient for prompt tuning.
\newblock In {\em ICCV}, 2023.

\bibitem{zhu2022cross}
Beier Zhu, Yulei Niu, Xian-Sheng Hua, and Hanwang Zhang.
\newblock Cross-domain empirical risk minimization for unbiased long-tailed classification.
\newblock In {\em AAAI}, 2022.

\bibitem{ZhuTSZ23}
Beier Zhu, Kaihua Tang, Qianru Sun, and Hanwang Zhang.
\newblock Generalized logit adjustment: Calibrating fine-tuned models by removing label bias in foundation models.
\newblock In {\em NeurIPS}, 2023.

\bibitem{SSP}
Xingyu Zhu, Beier Zhu, Yi~Tan, Shuo Wang, Yanbin Hao, and Hanwang Zhang.
\newblock Selective vision-language subspace projection for few-shot {CLIP}.
\newblock In {\em ACM MM}, 2024.

\end{thebibliography}

\newpage

\appendix

\section{Theoretical Analysis}\label{sec:proof}

\subsection{Proof of Eq.~(\ref{eq:sigma}): Estimation of Class Covariance from Marginal Second Order Moment}\label{sec:proof-LFPDL}
We first derive the second order moments for a multivariate Gaussian and then for a Gaussian mixture, corresponding to the marginal distribution of $\mathbb{P}(\x)$.

For a class $j$ with parameters $\z_j$ and $\Sigma$, the conditional probability density function is given by:
\begin{equation}
 \mathcal{N}(\x;\z_j,\Sigma)= \frac{1}{\sqrt{(2\pi)^d|\Sigma|}}\exp\{-\frac{1}{2}(\x - \z_j)^\top\Sigma^{-1}(\x - \z_j)\}
\end{equation}
The second order moment generating function for class $j$ is:
\begin{align}
 M_j&=\mathbb{E}_{\x\in \mathcal{C}_j}[\x\x^\top]=\int_{\x}\mathcal{N}(\x;\z_j,\Sigma)\x\x^\top \text{d}\x \\
    &=\frac{1}{\sqrt{(2\pi)^d|\Sigma|}}\int_{\x} \exp \{-\frac{1}{2}(\x - \z_j)^\top\Sigma^{-1}(\x - \z_j)\}\x\x^\top \text{d}\x \\
    &\overset{(a)}{=}\frac{1}{\sqrt{(2\pi)^d|\Sigma|}}\int_{\y} \exp \{-\frac{1}{2}\y^\top\Sigma^{-1}\y\}(\y+\z_j)(\y+\z_j)^\top \text{d}\y \\ 
    &=\frac{1}{\sqrt{(2\pi)^d|\Sigma|}}\int_{\y} \exp \{-\frac{1}{2}\y^\top\Sigma^{-1}\y\}(\y\y^\top+\underbrace{\y\z_j^\top+\z_j\y^\top}_{\text{vanish by symmetry}}+\z_j\z_j^\top) \text{d}\y \\
    &\overset{(b)}{=}\frac{1}{\sqrt{(2\pi)^d|\Sigma|}}\int_{\y} \exp \{-\frac{1}{2}\y^\top\Sigma^{-1}\y\}(\y\y^\top+\z_j\z_j^\top) \text{d}\y \\
    &\overset{(c)}{=}\z_j\z_j^\top + \frac{1}{\sqrt{(2\pi)^d|\Sigma|}}\int_{\y} \exp \{-\frac{1}{2}\y^\top\Sigma^{-1}\y\}(\y\y^\top) \text{d}\y. \label{eq:beforeinverseS}
\end{align}
$\overset{(a)}{=}$ holds as we change the integral variables $\y=\x-\z_j$. We have $\overset{(b)}{=}$ because the $\exp(\cdot)$ function is an even function of $\y$ and the factors $\y\z_j^\top$ and $\z_j\y^\top$ will vanish during integral by symmetry. For $\overset{(c)}{=}$, we take the term $\z_j\z_j^\top$ outside of the integral as they are constant. 

The covariance matrix $\Sigma$ and its inverse matrix $\Sigma^{-1}$ the can be expressed through an expansion in terms of its eigenvalues $\{\lambda_i\}_{i=1}^d$ and eigenvectors $\{\mathbf{u}_i\}_{i=1}^d$:
\begin{equation}
    \Sigma=\sum_{i=1}^d \lambda_i \mathbf{u}_i\mathbf{u}_i^\top, \quad \Sigma^{-1}=\sum_{i=1}^d \frac{1}{\lambda_i} \mathbf{u}_i\mathbf{u}_i^\top
\end{equation}

Similarly, we can decompose $\y$ using the set of eigenvectors:
$\y=\sum_{j=1}^de_j\mathbf{u}_j$, where $e_j=\mathbf{u}_j^T\y$. (We temporarily abuse the subscript $j$ here. It does \textit{not} represent class $j$ until we reach Eq.~\eqref{eq:getSigma}) We have the following expression:
\begin{align}
\y\y^\top&=\sum_{i=1}^d\sum_{j=1}^de_ie_j\mathbf{u}_i\mathbf{u}_j^\top \label{eq:yy}\\ \y^\top\Sigma^{-1}\y&=\sum_{i=1}^de_i\mathbf{u}_i^\top \sum_{k=1}^d \frac{1}{\lambda_k} \mathbf{u}_k\mathbf{u}_k^\top\sum_{j=1}^de_j\mathbf{u}_j \overset{(d)}{=}\sum_{k=1}^d(\frac{e_k}{\sqrt{\lambda_k}})^2 \label{eq:ySy}
\end{align}
We obtain $\overset{(d)}{=}$ due to the property of eigenvalues, \ie, $\mathbf{u}_i^\top\mathbf{u}_i=1$ and $\mathbf{u}_i^\top\mathbf{u}_j=0,\ \text{for}\ i\neq j$. Denote $U=[\mathbf{u}_1,...,\mathbf{u}_d]^\top$, we have $\e=U\y$. As the determinant $|U|=1$, the probability density after transformed remains unchanged: $\mathbb{P}(\e)=|U|^{-1}\mathbb{P}(\y)=\mathbb{P}(\y)$. Apply Eq.~\eqref{eq:yy} and Eq.~\eqref{eq:ySy} into Eq.~\eqref{eq:beforeinverseS}, we have:
\begin{align}
    &\frac{1}{\sqrt{(2\pi)^d|\Sigma|}}\int_{\y} \exp \{-\frac{1}{2}\y^\top\Sigma^{-1}\y\}(\y\y^\top) \text{d}\y \\
    =&\frac{1}{\sqrt{(2\pi)^d|\Sigma|}}\sum_{i=1}^d\sum_{j=1}^d\mathbf{u}_i\mathbf{u}_j^\top\int_{\e}\exp\{\sum_{k=1}^d-\frac{1}{2}(\frac{e_k}{\sqrt{\lambda_k}})^2\}e_ie_j\text{d}\e \\
    =&\frac{1}{\sqrt{(2\pi)^d|\Sigma|}}\sum_{i=1}^d\sum_{j=1}^d\mathbf{u}_i\mathbf{u}_j^\top\int_{\e}\prod_{k=1}^d \exp\{-\frac{1}{2}(\frac{e_k}{\sqrt{\lambda_k}})^2\}e_ie_j\text{d}\e \\
    \overset{(e)}{=}&\frac{1}{\sqrt{(2\pi)^d|\Sigma|}}\sum_{i=1}^d\mathbf{u}_i\mathbf{u}_i^\top\int_{e_i}\exp\{-\frac{1}{2}(\frac{e_i}{\sqrt{\lambda_i}})^2\}e_i^2\text{d}e_i \\
    \overset{(f)}{=}&\sum_{i=1}^d\mathbf{u}_i\mathbf{u}_i^\top\int_{e_i}\frac{1}{\sqrt{2\pi\lambda_i}} \exp\{-\frac{1}{2}(\frac{e_i}{\sqrt{\lambda_i}})^2\}e_i^2\text{d}e_i \\
    \overset{(g)}{=}&\sum_{i=1}^d\mathbf{u}_i\mathbf{u}_i^\top\lambda_i=\Sigma \label{eq:getSigma} 
\end{align}
For $\overset{(e)}{=}$, the terms $i\neq j$ disappear by symmetry similar to  $\overset{(b)}{=}$. We make use of $|\Sigma|=\prod_{i=1}^d\lambda_i$ for $\overset{(f)}{=}$. We have $\overset{(g)}{=}$ because we regard $e_i \sim \mathcal{N}(0,\sqrt{\lambda_i})$ and note that $\mathbb{E}[e_i^2]=\mathsf{var}[e_i]+\mathbb{E}[e_i]^2=\lambda_i+0=\lambda_i$. Combining Eq.~\eqref{eq:getSigma} with Eq.~\eqref{eq:beforeinverseS}, we have the second order moment for class $j$ is:
\begin{equation}
    M_j=\z_j\z_j^\top + \Sigma.
\end{equation}

Using a Gaussian mixture model with the priors $\{\pi_j\}_j^K$, $\mathbb{P}(\x)$ is given by:
\begin{equation}
    \mathbb{P}(\x)=\sum_{j=1}^K\pi_j\mathcal{N}(\x;\z_j,\Sigma).
\end{equation}
The second order moment for the marginal distribution $\mathbb{P}(\x)$ is:
\begin{align}
    M&=\mathbb{E}[\x\x^\top]=\int_\x \sum_{j=1}^K\pi_j\mathcal{N}(\x;\z_j,\Sigma)\text{d}\x\x^\top\x \\
    &=\sum_{j=1}^K\pi_j \int_\x \mathcal{N}(\x;\z_j,\Sigma)\x\x^\top\text{d}\x \\
    &=\sum_{j=1}^K\pi_j M_j=\sum_{j=1}^K\pi_j(\z_j\z_j^\top + \Sigma)\\
    &=\sum_{j=1}^K\pi_j\z_j\z_j^\top+(\sum_{j=1}^K\pi_j)\Sigma=\Sigma+\sum_{j=1}^K\pi_j\z_j\z_j^\top 
\end{align}
\hfill \qed

\subsection{Proof of Eq.~(\ref{eq:est_pi}): Estimation of the Priors of Gaussian Mixture Models}\label{sec:proof-pi}
The expectation of $\x$ is defined as:
\begin{equation}\label{eq:exp_x}
    \mathbb{E}[\x]=\int_\x \sum_{j=1}^K\pi_j\mathcal{N}(\x;\z_j,\Sigma)\x\text{d}\x= \sum_{j=1}^K\pi_j\int_\x\mathcal{N}(\x;\z_j,\Sigma)\x\text{d}\x= \sum_{j=1}^K\pi_j \z_j
\end{equation}
Denote $\bm{\pi}=[\pi_1,..,\pi_K]^\top$, $Z=[\z_1,..,\z_K]^T$, and the expectation of $\x$ as $\bm{\mu}$, Eq.~\eqref{eq:exp_x} can be rewrite as:
\begin{equation}
    \bm{\mu} = Z\bm{\pi}.
\end{equation}
Therefore, the priors can be solve by $\bm{\pi}=Z^{-1}\bm{\mu}$.
\hfill \qed

\subsection{Proof of Eq.~(\ref{eq:sg_label}): Parameters of our Learned Model}\label{sec:proof_sg_label}
The posterior of classes $\mathbb{P}(y|\x)$ can be expression as:
\begin{equation}
    \mathbb{P}(y|\x)=\frac{\mathbb{P}(\x|y)\mathbb{P}(y)}{\mathbb{P}(\x)} \propto \mathcal{N}(\x;\z_y,\Sigma)\pi_y. 
\end{equation}
To classify $\x$, we seek the class $y$ that maximizes this posterior. Since the term $\mathbb{P}(\x)$ does not depend on $y$, we can simplify our task to maximizing $\mathcal{N}(\x; \mu_y, \Sigma)\pi_y$. Taking natural logarithms gives:
\begin{align}
    \ln{\mathcal{N}(\x;\z_y,\Sigma)\pi_y} &=\ln{\frac{1}{\sqrt{(2\pi)^d|\Sigma|}}\exp\{-\frac{1}{2}(\x - \z_y)^\top\Sigma^{-1}(\x - \z_y)\}\pi_y} \\
    &=\ln{\frac{1}{\sqrt{(2\pi)^d|\Sigma|}}} -\frac{1}{2}(\x - \z_y)^\top\Sigma^{-1}(\x - \z_y)+ \ln{\pi_y} \label{eq:c1} \\
    &=c_1-\frac{1}{2}\x^T\Sigma^{-1}\x+\z_y^\top\Sigma^{-1}\x -\frac{1}{2}\z_y^T\Sigma^{-1}\z_y + c_2 \\
    &=\z_y^\top\Sigma^{-1}\x -\frac{1}{2}\z_y^T\Sigma^{-1}\z_y+ c  \label{eq:absort_all}\\
    &=\w_y^T\x+b_y + c \label{eq:wx_b}
\end{align}
The first term in Equation \eqref{eq:c1} is constant; we incorporate it using a constant $c_1$. Consider that most test benchmarks are generally class-balanced, we use a uniform prior $c_2$ to incorporate $\ln{\pi_y}$. In Eq.~\eqref{eq:absort_all}, we use $c$ to absorb all constant terms, including $c_1, c_2$ and $-\frac{1}{2}\x^T\Sigma^{-1}\x$. Let $\w_j = \hat{\Sigma}^{-1} \z_j$ and $b_j = -\frac{1}{2} \z_j^\top \w_j$, we get Eq.~\eqref{eq:wx_b}. 
\hfill \qed

\subsection{Proof of Eq.~(\ref{eq:debias}): Debiased Classifier for Downstream Data}\label{sec:debias}
\begin{proposition}
(Modified from Theorem 1 in~\cite{HongHCSKC21}). Let $\mathbb{P}_\mathsf{pt}(y|\x)$ and $\mathbb{P}_\mathsf{ds}(y|\x)$ be the distributions of the pre-train and downstream data, respectively.
Let $\beta_y=\mathbb{P}_\mathsf{pt}(y)$ and $\pi_y=\mathbb{P}_\mathsf{ds}(y)$ denote the priors of the pre-train and the downstream data, respectively.
Assume the likelihood $\mathbb{P}(\x|y)$ is unchanged between pre-train and downstream data, \ie, $\mathbb{P}(\x|y)=\mathbb{P}_\mathsf{pt}(\x|y)=\mathbb{P}_\mathsf{ds}(\x|y)$. If $f_\mathsf{pt}(\x)_y$ is the logit of class $y$ from the softmax model to estimate $\mathbb{P}_\mathsf{pt}(y|\x)$, then the estimated $\mathbb{P}_\mathsf{ds}(y|\x)$ is formulated as:
\begin{equation}
    \mathbb{P}_\mathsf{ds}(y|\x)=\softmax(f_\mathsf{pt}(\x)-\ln{\bm{\beta}} + \ln{\bm{\pi}})_y,
\end{equation}
where $\bm{\beta}=[\beta_1,...,\beta_K]$ and $\bm{\pi}=[\pi_1,...,\pi_K]$.
\end{proposition}

\begin{proof}
\begin{align}
    \mathbb{P}_\mathsf{ds}(y|\x)&=\frac{\mathbb{P}_\mathsf{ds}(\x|y)\mathbb{P}_\mathsf{ds}(y)}{\mathbb{P}_\mathsf{ds}(\x)}=\frac{\mathbb{P}_\mathsf{pt}(\x|y)\mathbb{P}_\mathsf{ds}(y)}{\mathbb{P}_\mathsf{ds}(\x)} \\
    &=\frac{\mathbb{P}_\mathsf{pt}(\x|y)\mathbb{P}_\mathsf{pt}(y)}{\mathbb{P}_\mathsf{pt}(\x)}\frac{\mathbb{P}_\mathsf{ds}(y)}{\mathbb{P}_\mathsf{pt}(y)}\frac{\mathbb{P}_\mathsf{pt}(\x)}{\mathbb{P}_\mathsf{ds}(\x)} \\
    &\overset{(a)}{=}\mathbb{P}_\mathsf{pt}(y|\x)\frac{\pi_y}{\beta_y}\frac{1}{Z}  =\softmax(f_\mathsf{pt}(\x))_y \frac{\pi_y}{\beta_y}\frac{1}{Z} \\
    &=\frac{\exp(f_\mathsf{pt}(\x)_y)}{Z\sum_{j=1}^K \exp(f_\mathsf{pt}(\x)_j)}\frac{\exp{(\ln{\pi_y})}}{\exp{(\ln{\beta_y})}} \\
    &=\frac{\exp(f_\mathsf{pt}(\x)_y-\ln{\beta_y}+\ln{\pi_y})}{Z\sum_{j=1}^K \exp(f_\mathsf{pt}(\x)_j)} \\
    &\overset{(b)}{=}\frac{\exp(f_\mathsf{pt}(\x)_y-\ln{\beta_y}+\ln{\pi_y})}{\sum_{j=1}^K \exp(f_\mathsf{pt}(\x)_j-\ln{\beta_j}+\ln{\pi_j})} \\
    &=\softmax(f_\mathsf{pt}(\x)-\ln{\bm{\beta}} + \ln{\bm{\pi}})_y.
\end{align}
For $\overset{(a)}{=}$, we denote the term that is not related to $y$ as $\frac{1}{Z}=\frac{\mathbb{P}_\mathsf{pt}(\x)}{\mathbb{P}_\mathsf{ds}(\x)}$. We derive $\overset{(b)}{=}$ from the requirement that $\mathbb{P}_\mathsf{ds}(y|\x)$, being a probability, must sum to 1 across all possible classes $y\in [K]$:
\begin{equation}
    \sum_{i=1}^K\mathbb{P}_\mathsf{ds}(i|\x)=\frac{\sum_{i=1}^K\exp(f_\mathsf{pt}(\x)_i-\ln{\beta_i}+\ln{\pi_i})}{Z\sum_{j=1}^K \exp(f_\mathsf{pt}(\x)_j)}=1.
\end{equation}
Therefore, we have $Z\sum_{j=1}^K \exp(f_\mathsf{pt}(\x)_j)=\sum_{i=1}^K\exp(f_\mathsf{pt}(\x)_i-\ln{\beta_i}+\ln{\pi_i})$. In our context, the pre-trained model $f_\mathsf{pt}$  is equivalent to our $f_\mathsf{f}$.
\end{proof}

\subsection{Proof of Eq.~(\ref{eq:qP=q}): Equation to Estimate Pre-training Priors}\label{sec:proveqP=q}
\begin{proposition}\label{prop:2}
    Let $s(\x)=[\mathbb{P}(Y=1|\x),...,\mathbb{P}(Y=K|\x)]^\top\in\mathbb{R}^K$ be the likelihood vector, $\s_j=\mathbb{E}_{\x|Y=j}[s(\x)]$ and $S = [\s_1,..., \s_K]\in \mathbb{R}^{K\times K}$. The pretraining prior $\bm{\beta} = [\beta_1,..., \beta_K]^\top \in \mathbb{R}^K$ must satisfy the  linear system:
\begin{equation}\label{eq:againq=Pq}
    (S-I)\bm{\beta} = \mathbf{0}.
\end{equation} 
\end{proposition}

\begin{proof}
    \begin{align}
        \beta_y&=\int_\x \mathbb{P}_\mathsf{pt}(\x)\mathbb{P}_\mathsf{pt}(y|\x)\text{d}\x=\int_\x \sum_{y'\in[K]}\mathbb{P}(\x|y')\beta_{y'}\mathbb{P}_\mathsf{pt}(y|\x)\text{d}\x \\
        &=\sum_{y'\in[K]} \beta_{y'} \int_\x \mathbb{P}(\x|y')\mathbb{P}_\mathsf{pt}(y|\x)\text{d}\x \\
        &=\sum_{y'\in[K]} \beta_{y'} \mathbb{E}_{\x|Y=y'}[\mathbb{P}_\mathsf{pt}(y|\x)]\\
        &=\sum_{y'\in[K]} \beta_{y'} \mathbb{E}_{\x|Y=y'}[s(\x)]_y, \\
        &=\sum_{y'\in[K]} S_{yy'} \beta_{y'} \label{eq:row-version}
    \end{align}
Note that Equation \eqref{eq:row-version} precisely represents the matrix multiplication given by:
\begin{equation}\label{eq:eigen}
\bm{\beta} = S\bm{\beta}
\end{equation}
By moving the RHS term to the LHS, Eq.~\eqref{eq:againq=Pq} is obtained.\end{proof}

\subsection{Power Method to Estimate Pretraining Priors}\label{sec:power}
The solution to Equation \eqref{eq:eigen} involves finding the eigenvector corresponding to the eigenvalue of $1$ for the matrix $S$. We can apply SVD decomposition to find the solution; however, we find that the results might be numerically unstable. Instead, we adopt power iteration from~\cite{mises1929praktische}. Like the Jacobi and Gauss-Seidel methods, the power method for approximating eigenvalues is iterative. We first initialize $\bm{\beta}_0=[\frac{1}{K},...,\frac{1}{K}]$ of a uniform distribution. Then, we perform the sequence:
\begin{align}
    \bar{\bm{\beta}}_t &= S \bm{\beta}_{t-1} \\
    \bm{\beta}_t &= \frac{\bar{\bm{\beta}_t}}{\|\bar{\bm{\beta}_t}\|_1}
\end{align}
We repeat the sequence until the relative change is small: $\|\bm{\beta}_t-\bm{\beta}_{t-1}\| < \epsilon$.



 \section{Details of ImageNet Variant Datasets}
\noindent{\textbf{ImageNet-V2~\cite{ImageNetV2}:}} sampling from the original ImageNet and including 10,000 images of 1,000 ImageNet categories.

\noindent{\textbf{ImageNet Sketch~\cite{ImageNetSketch}:}} including 138 50,000 images and covering 1,000 ImageNet categories.

\noindent{\textbf{ImageNet-R~\cite{ImageNetR}:}} containing renditions (\emph{e.g.}, art, cartoons, graffiti) for ImageNet classes, comprising 30,000 images from 200 ImageNet categories.

\noindent{\textbf{ImageNet-A~\cite{ImageNetA}:}} collecting real-world images that are misclassified by ResNet-50, totaling 7,500 images from 200 of ImageNet categories.

\noindent{\textbf{ObjectNet:~\cite{BarbuMALWGTK19}} including 50,000 test images with rotation, background, and viewpoint, and overlapping 113 classes with ImageNet.


\end{document}